\documentclass{l4dc2026}

\title[ORC]{Online Learning and Coverage of Unknown Fields\\Using Random-Feature Gaussian Processes}
\usepackage{times}
\usepackage{algpseudocode} 
\usepackage{algorithm,algcompatible}
\usepackage{appendix}
\usepackage{booktabs}

\newtheorem{assumption}{Assumption}

\DeclareMathOperator*{\argmax}{arg\,max}
\DeclareMathOperator*{\argmin}{arg\,min}

 \coltauthor{\Name{Ruijie Du} \Email{ruijied@uci.edu}\\
  \Name{Ruoyu Lin} \Email{rlin10@uci.edu}\\
  \Name{Yanning Shen} \Email{yannings@uci.edu}\\
  \Name{Magnus Egerstedt} \Email{magnus@uci.edu}\\
  \addr {Department of Electrical Engineering
  and Computer Science}, \\University of California, Irvine }

\begin{document}

\maketitle
\vspace{-1.5mm}
\begin{abstract}%
This paper proposes a framework for multi-robot systems to perform simultaneous learning and coverage of a domain of interest characterized by an unknown and potentially time-varying density function. To overcome the limitations of Gaussian Process (GP) regression, we employ Random Feature GP (RFGP) and its online variant (O-RFGP) which enables online and incremental inference. By integrating these with Voronoi-based coverage control and Upper Confidence Bound (UCB) sampling strategy, a team of robots can adaptively focus on important regions while refining the learned spatial field for efficient coverage. 
{The incremental update mechanism of O-RFGP naturally supports time-varying environments, allowing efficient adaptation without retaining historical data. Furthermore,} {to the best of our knowledge}, {we provide the first theoretical analysis of online learning and coverage through a regret-based formulation, establishing asymptotic no-regret guarantees in the time-invariant setting.
The effectiveness of the proposed framework is demonstrated through simulations with both time-invariant and time-varying density functions, along with a physical experiment
with a time-varying density function.}
\end{abstract}


\section{Introduction}
\label{Section:Introduction}


Environmental exploration and monitoring play a vital role in applications ranging from disaster response to ecological preservation, where multi-robot systems provide significant advantages due to their resilience 
and effectiveness in data collection and coverage across large domains \citep{rizk2019cooperative}. To guide the deployment of robots, a common approach to characterize the relative importance of the environment is through a density distribution function across the domain of interest, i.e., a spatial field. For example, in wildfire monitoring or fighting \citep{pham2017distributed,rickenbach2024active, lin2025heterogeneous}, the density function can represent the heat distribution
or fuel accumulation, while in sensor network deployment \citep{pimenta2010simultaneous,schwager2011eyes,belal2023understanding},
it can encode the spatial concentration of target-relevant information.

However, in many real-world scenarios, such a density distribution is not known a priori. Hence, the robots must perform environmental exploration to learn the underlying \textbf{unknown} spatial field. To this end, recent approaches leverage Bayesian inference to learn such unknown spatial fields while simultaneously deploying robots to efficiently cover the area \citep{PCC_cov, santos2021multi, nakamura2022decentralized}. This involves the joint process of exploration and exploitation, where cooperative robots gradually improve their knowledge of the environment through learning the unknown density while simultaneously covering relatively more important or informative areas.


Bayesian inference typically employs the \textbf{Gaussian Process (GP)} as a probabilistic surrogate that learns unknown functions with uncertainty quantification \citep{srinivas2012information, williams2006gaussian, calandriello2019gaussian}. {Combined with the Upper Confidence Bound (UCB) sampling technique \citep{auer2000using}, and Voronoi-based coverage control \citep{Lloyd, cortes2004coverage}, the GP-based methods proposed in \citet{PCC_cov, santos2021multi, nakamura2022decentralized} enable simultaneous learning of an unknown density function and 
 coverage over the learned spatial field.}
However, these previous works
rely on the assumption that the density function is time-invariant, which limits their applicability in dynamic scenarios such as wildfire monitoring, where the density function is time-varying. Building upon \citet{nakamura2022decentralized}, \citet{pratissoli2025distributed}
represents a step toward the open problem of learning and coverage of unknown, time-varying density functions. 
{However, even setting aside the lack of theoretical analysis or formal performance guarantees, the proposed GP-based method in \citet{pratissoli2025distributed} 
faces difficulties in achieving accurate time-varying density estimation in simulation, 
Moreover, the physical experiment was conducted with a 
time-invariant density with a single abrupt switch instead of a continuously evolving one.
Hence, its effectiveness under 
time-varying density remains unverified.}


{A key limitation of GP-based methods lies in their cubic computational complexity with respect to the number of training data \citep{williams2006gaussian}.} 
To tackle this challenge, we leverage the \textbf{random feature-based Gaussian process (RFGP)} to alleviate the computational complexity. 
RFGP is an approximation method that enables scalable inference in kernel-based learning and Gaussian processes \citep{rahimi2007random,shen2019random,calandriello2019gaussian,mutny2018efficient,wang2018batched,lu2020ensemble}. RFGP brings significant improvements in computational cost with acceptable or even negligible performance loss. 
However, RFGP still requires all collected measurements to be associated with outputs, which limits its applicability in learning time-varying density functions. To this end, we resort to 
\textbf{Online RFGP (O-RFGP)}, 
which updates incrementally once new training data becomes available, {enabling real-time adaptation to changing environments. 
The lightweight computation and incremental update mechanism make O-RFGP particularly suitable for time-varying function learning.
Based on the efficient and 
online updates obtained via O-RFGP, our proposed framework enables real-time learning and coverage of time-varying spatial fields.
}
The main contributions of this paper are summarized as follows. 

{
\noindent\textbf{c1)}
We propose a simultaneous learning and coverage framework
for exploring and monitoring of unknown fields.
With its incremental and lightweight update, the proposed framework is capable of real-time operation and well-suited for deployment on physical robotic platforms.

\noindent\textbf{c2)}
To the best of our knowledge, our framework offers the first analytical result on online learning and coverage through a regret-based analysis in the time-invariant case. 
Under mild assumptions, it is proved that the proposed algorithm has asymptotic no-regret with respect to the coverage obtained with the corresponding ground truth
density. 

\noindent\textbf{c3)}
{The performance of the proposed framework is evaluated through simulations with both time-invariant and time-varying densities. And its effectiveness for real-time deployment is further demonstrated in a physical experiment with a time-varying density.}
The results of both simulations and the physical experiment provide encouraging evidence that the proposed framework can effectively maintain its performance and stability in time-varying settings.
}



\section{Problem Formulation}

Consider a group of $n$ planar robots operating in a convex domain of interest $\mathcal{D} \subset \mathbb{R}^2$. 
The position of robot $i \in N \triangleq \{1,\dots,n\}$ is denoted as ${x}^i \in \mathcal{D}$, and the positions of all robots are grouped into a matrix $\mathbf{x} \triangleq \left[ {x}^1, \cdots, {x}^n \right]^\top \in \mathbb{R}^{n \times 2}$. 
The problem of coverage control concerns designing strategies for a group of robots to achieve optimal coverage of $\mathcal{D}$ associated with a density distribution. Specifically, the density distribution function, $f(x): \mathbb{R}^2 \rightarrow \mathbb{R}_{>0}$, encodes the relative importance 
of a location $x \in \mathcal{D}$ in the sense that the higher the value of $f(x)$, the more important 
the location $x$ is. Typically, 
the subdomain (or region of dominance) that robot $i$ is in charge of is given by the Voronoi cell, $\mathcal{V}^i \triangleq \big\{ {q}\in \mathcal{D}: \lVert{q}- {x}^i\rVert_2 \leq \lVert {q}-{x}^j\rVert_2,\,\, \forall j \neq i \in N \big\}$.
When the density function $f$ is known, an optimal coverage can be asymptotically achieved by each robot~$i$ following the gradient flow of the locational cost function
\begin{align}
     \ell(\mathbf{x}, f) = \sum_{i=1}^n \int_{{x}\in \mathcal{V}^i} \! \lVert {x}-{x}^i \rVert_2^2 \,f({x}) \, d{x}
    \label{location_cost_objective}
\end{align}
with respect to $x^i$ (\citet{cortes2004coverage}), resulting in a centroidal Voronoi tessellation (CVT),
$\mathbf{c}^i = \frac{\int_{\mathcal{V}^i}\! x f({x}) \,d{x}}{\int_{\mathcal{V}^i}\!f({x})\,d{x}}, \,\, \forall i \in N,$
where $\mathbf{c}^i$ is the centroid (i.e., center of mass) of $\mathcal{V}^i$ with respect to the density function $f$. We denote a CVT by the matrix $\mathbf{c} \triangleq \left[\mathbf{c}^1, \cdots, \mathbf{c}^n \right]^\top \in \mathbb{R}^{n\times 2}$.

However, in practice, as discussed in Section \ref{Section:Introduction}, the density function $f$ can be unknown a priori. Then,
the group of robots needs to achieve an optimal coverage of
a surrogate density function learned from their measurements. Specifically, at position ${x}^i$, robot $i$ takes the measurement
\begin{align}
    y^i = f({x}^i)+ \epsilon,
    \label{eqn:measurementnoise}
\end{align}
where $\epsilon \sim \mathcal{N}  (0,\sigma^2)$. As a robot typically takes independent measurements
at each time using the same sensor, the measurement noise $\epsilon$ in \eqref{eqn:measurementnoise} is independent and identically distributed (i.i.d.) across time and space (\citet{PCC_cov}).

Let all variables related to time $t\in \{1,\dots, T\}$ be affixed with a subscript $(\cdot)_t$. For example, $\mathbf{x}_t \triangleq \left[ {x}_t^1, \cdots, {x}_t^n \right]^\top$ denotes the locations of all the $n$ robots at time $t$. As each robot takes a measurement $y^i_t$ at time $t$, all measurements at time $t$ are grouped in a vector $\mathbf{y}_t \triangleq \left[ y^1_t, \cdots, y^n_t \right]^\top \in \mathbb{R}^{n}$. We use the subscript $(\cdot)_{:t}$ to indicate the variables with the range of time steps from the beginning to $t$. For example,
the positions of all the $n$ robots from $t = 1$ to $t = T$
are grouped into the matrix $\mathbf{x}_{:T} \in \mathbb{R}^{nT\times 2}$ and their corresponding measurements 
are denoted as the vector $\mathbf{y}_{:T} \in \mathbb{R}^{nT}$. Let $f_t({x})$ denote the time-varying density value at location ${x}$ at time $t$, whereas $f({x})$ represents the time-invariant density function at location ${x}$ whose value remains the same over time.

As the true density function is unknown, the robots need to simultaneously learn and update the surrogate density function $\hat{f}_t$ at each time step based on collected training data, i.e., measurements from all robots, and follow a control law to efficiently cover the domain associated with $\hat{f}_t$. The centroid of $\mathcal{V}_t^i$ with respect to $\hat{f}_t$ becomes 
$\mathbf{c}^i_t = \frac{\int_{\mathcal{V}_t^i} {x} \hat{f}_t({x}) \,d{x}}{\int_{\mathcal{V}_t^i}\hat{f}_t({x})\,d{x}}$, 
and the control law for the team of robots is given by 
\citet{cortes2004coverage} 
\begin{align}
\dot{\mathbf{x}}_t 
= 
\kappa (\mathbf{c}_t - \mathbf{x}_t),
\label{eq:tv_ctrl}
\end{align}
where 
$\mathbf{c}_t \triangleq \left[ \mathbf{c}^1_t, \cdots, \mathbf{c}^n_t \right]^\top$, and $\kappa >0$ is a control gain.





\section{Density Learning
} \label{sec:density-learn}
In this section, we detail the density learning methods. The unknown density function $f(\cdot)$ is modeled as the realization of a GP, specified by a mean function, $\mu(\cdot): \mathcal{D} \rightarrow \mathbb{R}$, and a covariance (kernel) function, $k(\cdot, \cdot): \mathcal{D} \times \mathcal{D} \rightarrow \mathbb{R}$. In particular, 
$\sigma^2({x}) = k({x}, {x})$. We assume that the structures of $\mu$ and $k$ are characterized by certain hyperparameters $\rho^*$ and $\tau^*$ \citep{srinivas2009gaussian},
\begin{align}
    f(\cdot) \sim \mathcal{GP}(\mu(\cdot;\rho^*),k(\cdot,\cdot;\tau^*)). \label{eq:true_density}
\end{align}
For time-varying density, \eqref{eq:true_density} transforms into
    $f_t(\cdot) \sim \mathcal{GP}(\mu(\cdot;\rho_t^*),k(\cdot,\cdot;\tau_t^*))$, 
where $\rho^*_t$ and $\tau_t^*$ are hyperparameters at $t$. 


To learn an unknown function $f$ that has a GP prior, 
the noisy measurements $\mathbf{y}_{:t}$ at positions $\mathbf{x}_{:t}$ are leveraged. Then, the learned density value $\hat{f}_t$ at location ${x} \in \mathcal{D}$ at time $t$ can be represented by
\begin{align}
    \hat{f}_t({x}) &\sim \mathcal{GP}(\mu_t({x}),k_t({x},{x})),
\end{align}
where $\mu_t$ and $k_t$ are mean and covariance functions calculated based on $\mathbf{y}_{:t}$ and $\mathbf{x}_{:t}$. In the following content, we will introduce i) the GP-based
approach in Section \ref{sec:GP}, 
ii) the RFGP-based
approach in Section \ref{sec:RFGP}, and iii) the 
O-RFGP-based approach in Section \ref{sec:O-RFGP}, to calculate $\mu_t$ and $k_t$.


\begin{table}[h]
\vspace{-2mm}
\caption{Computational complexity and memory cost of GP, RFGP, and O-RFGP. 
}
\centering
\begin{tabular}{  |c| c| c| c|} 
\hline
Method& GP & RFGP  & O-RFGP \\ 
\hline
Computational complexity  & $\mathcal{O}(T^3)$ & $\mathcal{O}(TD^2)$ & $\mathcal{O}(D^2)$ \\ 
\hline
Memory cost & $\mathcal{O}(T^2)$ & $\mathcal{O}(TD)$ & $\mathcal{O}(D^2)$ \\ 
\hline
\end{tabular}
\label{table:computation_complex_memory_cost}
\end{table}

\subsection{Gaussian Process (GP)} \label{sec:GP}
Assuming the mean function is a function 
parameterized by {$\rho_t$} and the kernel function parameterized by {$\tau_t$},
the corresponding $\mu^{\text{GP}}$ and $k^{\text{GP}}$ at time step $t$ are defined as \citet{williams2006gaussian}
\begin{align}
    \mu^{\text{GP}}_{t}({x}) & = \mu({x}; \rho_t) +{\mathbf{k}_{t}}^\top (\mathbf{K}_t+ \sigma^2 \mathbf{I})^{-1} (\mathbf{y}_{:t-1}- \mu_{:t-1}(;\rho_t)), \label{eq:GP_mean}\\
    k_{t}^{\text{GP}} ({x},{x};\tau_t) &= k({x},{x}; \tau_t) - {\mathbf{k}_{t}}^\top (\mathbf{K}_t+ \sigma^2 \mathbf{I})^{-1} \mathbf{k}_{t}, \label{eq:GP_var}
\end{align}
where $\mu_{:t-1}(;\rho_t)\triangleq [\mu(\mathbf{x}_1;\rho_t), \cdots,\mu(\mathbf{x}_{t-1};\rho_t)]^\top$, $\mathbf{k}_{t} \in \mathbb{R}^{nt}$ is the vector with the {$((j-1)n+i)$}th entry as 
 $   k({x}^i_j, {x};\rho^t)$, 
and $\mathbf{K}_t \in \mathbb{R}^{nt \times nt}$ is the covariance matrix with the {$((j-1)n+i,(j'-1)n+i')$}th entry as
    $ k({x}^i_j, {x}^{i'}_{j'};\rho^t)$, 
    {$\forall i,i' \in \{1,\cdots,n\}$,} 
    {$\forall j,j' \in \{1,\cdots,t\}$}.
Specifically, at time step $t$, the measurements $\mathbf{y}_{:t}$ at locations $\mathbf{x}_{:t}$ accumulated by the $n$ robots are used to learn the density function. According to \citet{williams2006gaussian}, the parameters $\rho_t$ and $\tau_t$ can be obtained by 
\begin{align*}
    \rho_t, \tau_t = \argmin_{\rho,\tau} &[(\mu_{:t}(;\rho)-\mathbf{y}_{:t})^\top (K_t(;\tau) +\sigma^2 \mathbf{I})^{-1}(\mu_{:t}(;\rho)-\mathbf{y}_{:t})] + \log \left| K_t(;\tau) + \sigma^2 \mathbf{I} \right|,
\end{align*}
where $|\cdot|$ denotes the determinant of a matrix. Note that the
choices of the mean and covariance functions are not unique \citep{williams2006gaussian}.
{For example, in the case where the linear model and Radial Basis Function (RBF) 
kernel are used, we have:} 
$   \mu_t ({x};\rho_t)  \triangleq \rho_t^\top{x}$, and $k ({x},{x}';\tau_t) \triangleq \exp{\left(-\frac{\lVert {x}-{x}'\rVert_2^2}{2\tau_t^2}\right)}$.
{Since GP needs the memory of saving
the covariance matrix $K_t$ and requires the computation of
the inverse of $(\mathbf{K}_t+ \sigma^2 \mathbf{I})$, at time $T$, the memory cost is $\mathcal{O}(T^2)$ and the computational complexity is $\mathcal{O}(T^3)$, as shown in Table \ref{table:computation_complex_memory_cost}.}
As the number of samples increases, updating the mean and kernel functions of the GP becomes computationally intensive, limiting the applicability of GP in real-time robotic applications.

\subsection{Random Feature Gaussian Process (RFGP)} \label{sec:RFGP}
Random feature Gaussian Process (RFGP) provides an efficient way to scale GP by approximating the kernel using random features. It is more suitable for applications requiring GPs but constrained by computational resources, such as physical robotics implementations in real-time. 

We define the real-valued $2D$-dimensional random feature (RF) vector as
\begin{align}
    \phi_{\mathbf{v}}({x}) \triangleq  \frac{1}{\sqrt{D}} [\sin(\mathbf{v}_1^\top {x}), \cos(\mathbf{v}_1^\top {x}) \dots, \sin(\mathbf{v}_D^\top {x}), \cos(\mathbf{v}_D^\top {x})]^\top . 
\end{align}
{where $D $ denotes the number of random features and $\mathbf{v}_i$ is a random vector sampled from the spectral distribution of the kernel \citep{shen2019random}.}
And we define $\phi_{\mathbf{v}}(\mathbf{x}) \triangleq [\phi_{\mathbf{v}}({x^1}), \cdots, \phi_{\mathbf{v}}({x^n})]$.
A low-rank approximation of $\mathbf{K}_t$ can be achieved by using the standardized shift-invariant $\Bar{k}$, whose inverse Fourier transform can be approximated by \citet{lu2020ensemble} 
 $   \Bar{k}({x},{x}' ) \approx \frac{1}{D} \sum_{i=1}^D e^{j \mathbf{v}_i^\top ({x}-{x}')} = \phi_{\mathbf{v}}^\top({x}) \phi_{\mathbf{v}}({x}'),$
and $k = \sigma^2_{\theta} \Bar{k}$, where $\sigma^2_{\theta}$ is the magnitude of $k$.
A low-rank ($2D$) approximation of $\mathbf{K}^t$ is shown as
    $\Bar{\mathbf{K}}_t= \sigma^2_{\theta} \mathbf{\Phi}_t {\mathbf{\Phi}_t}^\top$, 
where $\mathbf{\Phi}_t \triangleq [\phi_{\mathbf{v}}(\mathbf{x}_1), \cdots, \phi_{\mathbf{v}}(\mathbf{x}_t)]^\top $.
Then, the $\mu^{\text{RF}}$ and $k^{\text{RF}}$ at time step $t$ can be calculated by 
    $\mu_t^{\text{RF}}({x})  = \phi_{\mathbf{v}}^\top({x}) \left( {\mathbf{\Phi}_t}^\top\mathbf{\Phi}_t + \frac{\tau_{\text{RF}} \mathbf{I}_{2D}}{\sigma_{\theta}^2}  \right)^{-1} {\mathbf{\Phi}_t}^\top \mathbf{y}_{:t},  ~~
    k_t^{\text{RF}} ({x},{x})
    = \phi_{\mathbf{v}}^\top({x}) \left( \frac{{\mathbf{\Phi}_t}^\top\mathbf{\Phi}_t}{\tau_{\text{RF}}} + \frac{\mathbf{I}_{2D}}{\sigma_{\theta}^2}  \right)^{-1} \phi_{\mathbf{v}}({x}).$
Let $\mathbf{\theta}_t = \left( {\mathbf{\Phi}_t}^\top\mathbf{\Phi}_t + \frac{\tau_{\text{RF}} \mathbf{I}_{2D}}{\sigma_{\theta}^2}  \right)^{-1} {\mathbf{\Phi}_t}^\top \mathbf{y}_{:t}$, $ \mathbf{\Sigma}_t = \left( \frac{{\mathbf{\Phi}_t}^\top\mathbf{\Phi}_t}{\tau_{\text{RF}}} + \frac{\mathbf{I}_{2D}}{\sigma_{\theta}^2}  \right)^{-1}$,
$\mu_t^{\text{RF}}$ and $k_t^{\text{RF}}$ 
at time step $t$ can be rewritten as
\begin{align}
    \mu_t^{\text{RF}}({x}) = \phi_{\mathbf{v}}^\top({x}) \mathbf{\theta}_t, 
    ~~k_t^{\text{RF}} ({x},{x}) = \phi_{\mathbf{v}}^\top({x}) \mathbf{\Sigma}_t \phi_{\mathbf{v}}({x}). \label{eq:RFGP}
\end{align}
Then, $\mathbf{\theta}_t$ and $\mathbf{\Sigma}_t$ can be considered as the parameters of RFGP. For example, if $\mu=\rho^\top {x}$, ideally, $\phi_{\mathbf{v}}^\top({x}) \mathbf{\theta}^* = \rho^\top {x}$, where $\mathbf{\theta}^*$ denotes the optimal model parameters. As such, 
{RFGP hugely reduces computational complexity and memory cost by using a fixed dimension RF vector to approximate the kernel, i.e., at time $T$, the memory cost is $\mathcal{O}(TD)$ and the computational complexity is $\mathcal{O}(TD^2)$, where $D$ is much smaller than $T$, 
as summarized in Table \ref{table:computation_complex_memory_cost}.}

\subsection{Online Random Feature Gaussian Process (O-RFGP)} \label{sec:O-RFGP}
The RF-based function approximant $\mu^{\text{RF}}$ and $k^{\text{RF}}$ readily accommodates online operation \citep{lu2020ensemble}, which is necessary for many time-critical applications, such as simultaneous learning and covering time-varying density functions.
The online approach allows interactive learning, i.e., the prediction at time step $t$ receives $\mathbf{x}_t$ at the beginning of slot $t$ and then updates after receiving measurements
at the end of slot $t$. Specifically, the online learning procedure is
shown as follows:

\begin{enumerate}
    \item Predict mean and covariance at location $\mathbf{x}_t$ as 
    \begin{align}
        \mu^{\text{O}}_{t}(\mathbf{x}_{t}) =  \phi_{\mathbf{v}}^\top(\mathbf{x}_{t}) \mathbf{\theta}_{t-1} , 
        ~~{\sigma^2_{t}}(\mathbf{x}_{t}) = k^{\text{O}}_{t}(\mathbf{x}_{t},\mathbf{x}_{t}) =\phi_{\mathbf{v}}^\top(\mathbf{x}_{t}) \mathbf{\Sigma}_{t-1} \phi_{\mathbf{v}}(\mathbf{x}_{t}).
        \label{eq:O-RFGP}
    \end{align}
    \item Upon receiving $\mathbf{y}_{t}$, update $\mathbf{\theta}_t$ and $\mathbf{\Sigma}_t$ as
    \begin{align}
        \mathbf{\theta}_{t} &= \mathbf{\theta}_{t-1} + \sigma_{t}^{-2}(\mathbf{x}_{t}) \mathbf{\Sigma}_t \phi_{\mathbf{v}}(\mathbf{x}_{t}) (\mathbf{y}_{t}- \mu^{\text{O}}_{t}(\mathbf{x}_{t})) , \label{eq:O-RFGP_theta}\\ 
        \mathbf{\Sigma}_{t} &= \mathbf{\Sigma}_{t-1} - \sigma_{t}^{-2}(\mathbf{x}_{t}) \mathbf{\Sigma}_{t-1} \phi_{\mathbf{v}}(\mathbf{x}_{t}) \phi_{\mathbf{v}}^\top(\mathbf{x}_{t}) \mathbf{\Sigma}_{t-1} . \label{eq:O-RFGP_Sigma}
\end{align}
\end{enumerate}

The O-RFGP can significantly save data storage and further reduce the computational complexity in the sense that at each time step, the computational complexity and the memory cost are both $\mathcal{O}(D^2)$, as summarized in Table \ref{table:computation_complex_memory_cost}. This iterative update mechanism enables O-RFGP to adapt to time-varying function learning by relying solely on the current measurements. Meanwhile, the low computational complexity of O-RFGP allows real-time robotic implementations.

\section{Learning-Aided Coverage Algorithm}
\label{Sec:ORC}

\begin{algorithm}[t] 
    \caption{
    O-RFGP-Aided Coverage Algorithm (ORC)
    }
    \begin{algorithmic}[1]
    
    \STATE  \textbf{Initialization:} \\
    Non-decreasing sequence $\{\beta_t\}_{t=1}^T$.\\
    Random locations $\mathbf{x}_0 = \left[ \mathbf{x}^1_0, \cdots, \mathbf{x}^n_0 \right]^\top$.\\
    Compute $\mu_0(\cdot)$ and $\sigma_0(\cdot)$.
    \FOR{$t=1 \dots T$}
        \STATE Compute the surrogate $\hat{f}_t$ for all ${x}\in \mathcal{D}$ by 
        $\hat{f}_t({x}) = \mu_{t-1}({x}) - \sqrt{\beta_t} \sigma_{t-1}({x}).$
        \STATE {Run control law \eqref{eq:tv_ctrl}
        for robots to reach $\mathbf{x}_t$, {i.e., a discrete integration step of \eqref{eq:tv_ctrl}}.
        }

        \STATE Determine the location  $ \tilde{{x}}_t = \argmax_{{x}\in \mathcal{D}} \sigma_{t-1}({x}),$ where to take additional measurement.
       
        \STATE 
        Obtain the measurements
        $ y_t^i = f_t({x}_t^i)+ \epsilon, ~~ \forall i \in \{1,\cdots,n\},~~
         \Tilde{y}_t = f_t(\tilde{{x}}_t)+ \epsilon.$
        \STATE Update $\mu_t(\cdot)$ via \eqref{eq:O-RFGP_theta} and $\sigma_t(\cdot)$ via \eqref{eq:O-RFGP_Sigma} 
        \ENDFOR
    \end{algorithmic}
        \label{alg:SLAC}
\end{algorithm}




In this section, first, we propose an algorithm that enables robots to simultaneously learn an unknown spatial field and efficiently cover the most recently acquired estimate of it. The algorithm, which integrates GP-based learning and Voronoi-based coverage control, combined with the UCB sampling strategy, guides the robots’ sensing and movement, incrementally driving them toward a CVT with respect to the evolving learned density function. Specifically, this is achieved by alternating between a prediction step, which allows the multi-robot team to select new sampling points, and a correction step, which allows it to update the learned density to be covered. A detailed outline of the proposed algorithm is presented in Algorithm \ref{alg:SLAC}. 
Note that the estimation step is based on O-RFGP, but can be directly replaced by GP or RFGP presented in Section \ref{sec:density-learn}. 
{Note that the estimation step is based on O-RFGP but can be directly replaced with GP or RFGP as presented in Section \ref{sec:density-learn}.
We define the cumulative regret as the sum of the  $r_t$ at time step $t$ as the difference between the locational cost with current locations $\mathbf{x}_t$ and the optimal locational cost
\begin{align} \label{cum_regret}
   R_T \triangleq \sum_{t=1}^T r_t, \; \text{with}~~ r_t = \ell(f,\mathbf{x}_t) - \min_{\mathbf{x}}\ell(f,\mathbf{x}) \geq 0.
\end{align}
A small cumulative regret $R_T$ indicates that the achieved coverage cost $\ell(f,\mathbf{x}_t)$ remains close to the optimal cost, meaning that the robots attain coverage performance comparable to the ideal case with full knowledge of the density function.
The theoretical guarantee of sub-linear regret in the time-invariant setting is formally established in Theorem~\ref{theorem1}, where the density function $f$, as in \eqref{eq:true_density}, does not vary over time.
The detailed regret analysis is provided in Appendix~\ref{sec:regret}.
The empirical results in Section \ref{sec:experiments} and Appendix \ref{sec:additionalexp} demonstrate that the cumulative regret remains sub-linear and is influenced by the same key parameters ($D$ and $\beta_t$), which suggests that similar principles govern the performance of O-RFGP in the time-varying setting.}






\section{Experiments} \label{sec:experiments}
{In this section, we demonstrate the effectiveness of the proposed framework ORC through both simulations and a physical experiment. The simulations are designed to evaluate performance under both time-invariant and time-varying spatial fields, thereby examining the framework’s capability to achieve accurate learning and efficient coverage in static as well as dynamic environments. 
We further conduct a hardware-in-the-loop experiment to demonstrate the
computational efficiency and 
effectiveness
of O-RFGP for deployment of physical multi-robot systems in real-time.
}

\begin{figure}[t]
\centering
\vspace{-0.1cm}
\subfigure[]{
\begin{minipage}[b]{0.2\textwidth}
\includegraphics[width=1\textwidth]{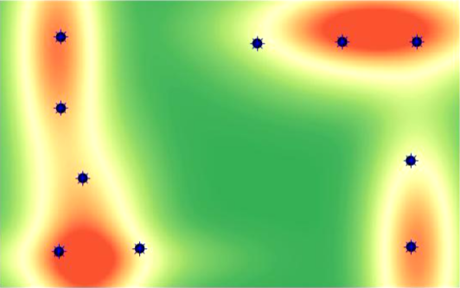}
\end{minipage}
\label{True_TI}
}
\subfigure[]{
\begin{minipage}[b]{0.2\textwidth}
\includegraphics[width=1\textwidth]{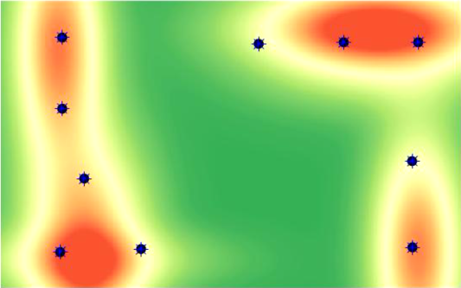}
\end{minipage}
\label{GP_TI}
}
\subfigure[]{
\begin{minipage}[b]{0.2\textwidth}
\includegraphics[width=1\textwidth]{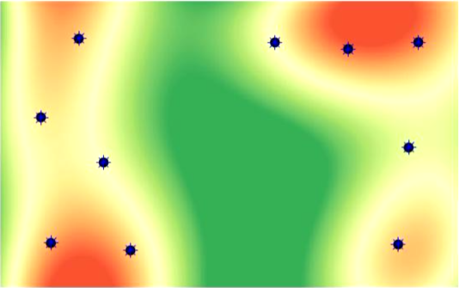}
\end{minipage}
\label{RFGP_TI}
}
\subfigure[]{
\begin{minipage}[b]{0.2\textwidth}
\includegraphics[width=1\textwidth]{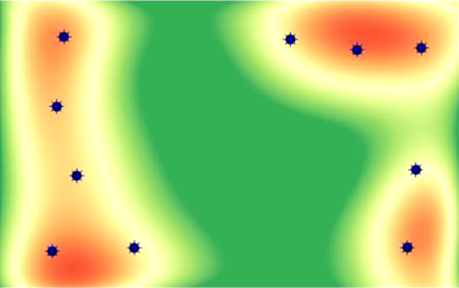}
\end{minipage}
\label{ORFGP_TI}
}
\caption{The final configuration (at time step $500$) of ten robots covering the time-invariant density distribution that is: (a) the ground truth, (b) learned using GP, (c) learned using RFGP, and (d) learned using O-RFGP. Areas in red (green) indicate higher (lower) importance. }
\label{CC_TI_GP_RFGP_ORFGP}
\vspace{-0.2cm}
\end{figure}

\begin{figure}[t]
\centering
\vspace{-1mm}
\subfigure[]{
\begin{minipage}[b]{0.3\textwidth}
\includegraphics[width=1\textwidth]{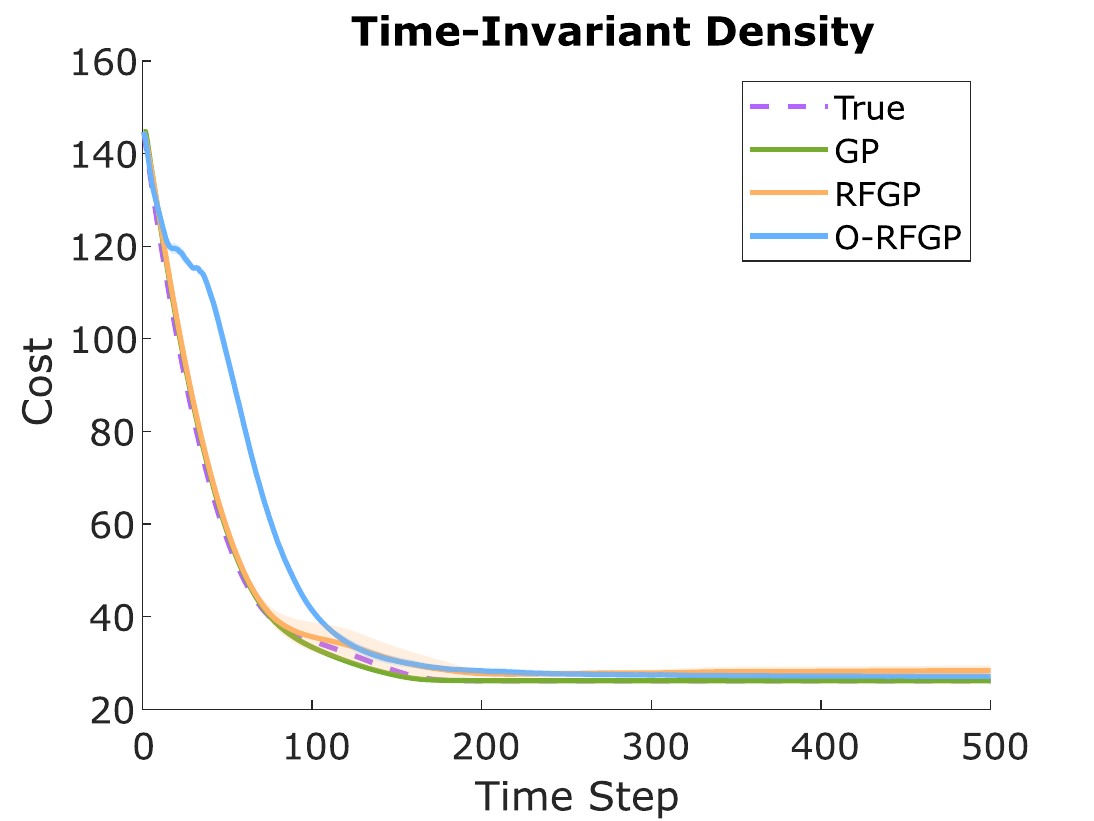}
\end{minipage}
\label{Cost_TI}
}
\subfigure[]{
\begin{minipage}[b]{0.3\textwidth}
\includegraphics[width=1\textwidth]{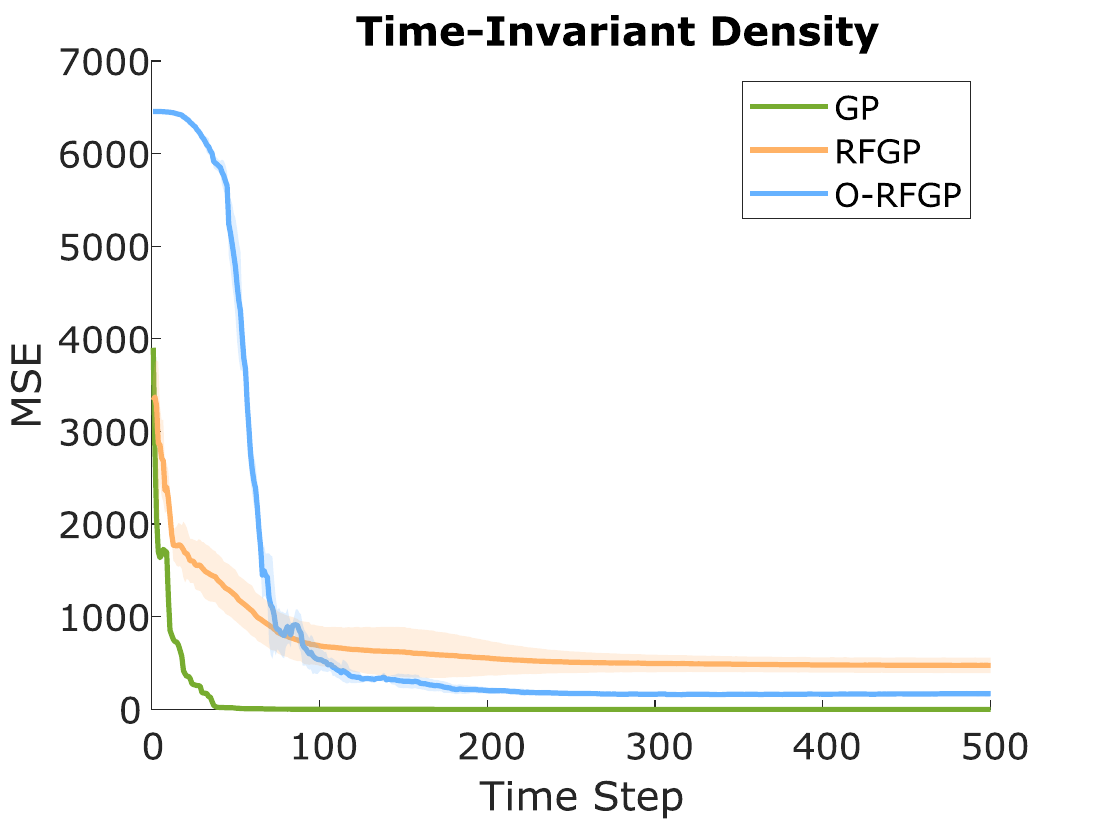}
\end{minipage}
\label{MSE_TI}
}
\caption{The evolution of the locational cost and the MSE between the learned density and the true density with respect to time steps. 
}
\label{Data_TI}
\vspace{-3mm}
\end{figure}




\begin{figure*}[t]
\centering
\vspace{-1mm}
\subfigure[$T=10$]{
\begin{minipage}[b]{0.18\textwidth}
\includegraphics[width=1\textwidth]{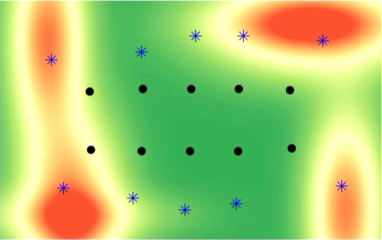}
\end{minipage}
\label{TV_true_iter10}
}
\subfigure[$T=30$]{
\begin{minipage}[b]{0.18\textwidth}
\includegraphics[width=1\textwidth]{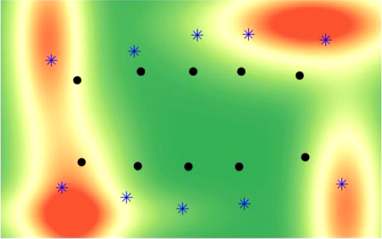}
\end{minipage}
\label{TV_true_iter30}
}
\subfigure[$T=100$]{
\begin{minipage}[b]{0.18\textwidth}
\includegraphics[width=1\textwidth]{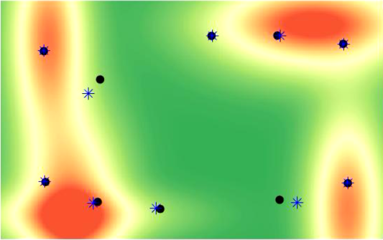}
\end{minipage}
\label{TV_true_iter100}
}
\subfigure[$T=1200$]{
\begin{minipage}[b]{0.18\textwidth}
\includegraphics[width=1\textwidth]{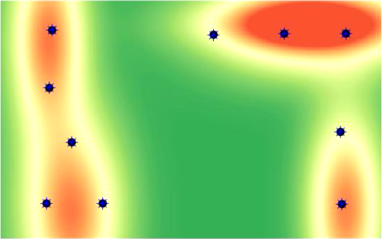}
\end{minipage}
\label{TV_true_iter1200}
}
\subfigure[$T=2000$]{
\begin{minipage}[b]{0.18\textwidth}
\includegraphics[width=1\textwidth]{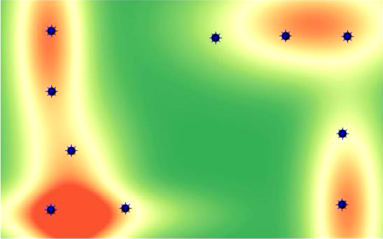}
\end{minipage}
\label{TV_true_iter2000}
}
\caption{The heatmaps of ten robots covering the true time-varying density distribution at different time steps, where the black dots represent the positions of robots and the blue asterisks represent the centers of mass of their Voronoi cells. 
}
\label{TV_True}
\vspace{-2mm}
\end{figure*}

\begin{figure*}[t]
\centering
\vspace{-1mm}
\subfigure[$T=10$]{
\begin{minipage}[b]{0.18\textwidth}
\includegraphics[width=1\textwidth]{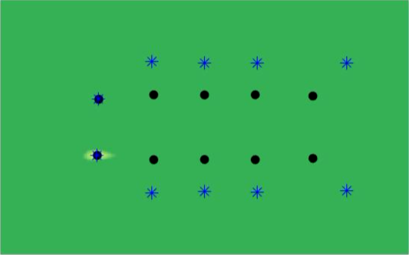}
\end{minipage}
\label{TV_ORFGP_sim_iter10}
}
\subfigure[$T=30$]{
\begin{minipage}[b]{0.18\textwidth}
\includegraphics[width=1\textwidth]{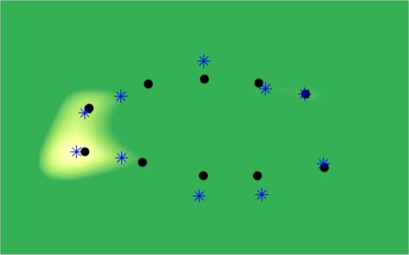}
\end{minipage}
\label{TV_ORFGP_sim_iter30}
}
\subfigure[$T=100$]{
\begin{minipage}[b]{0.18\textwidth}
\includegraphics[width=1\textwidth]{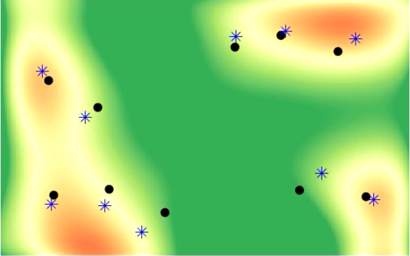}
\end{minipage}
\label{TV_ORFGP_sim_iter100}
}
\subfigure[$T=1200$]{
\begin{minipage}[b]{0.18\textwidth}
\includegraphics[width=1\textwidth]{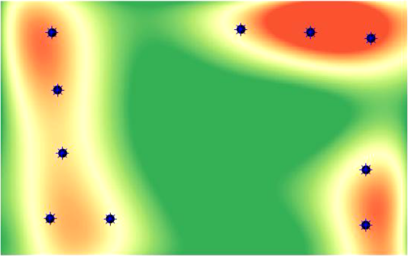}
\end{minipage}
\label{TV_ORFGP_sim_iter1200}
}
\subfigure[$T=2000$]{
\begin{minipage}[b]{0.18\textwidth}
\includegraphics[width=1\textwidth]{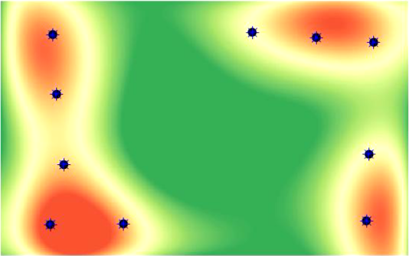}
\end{minipage}
\label{TV_ORFGP_sim_iter2000}
}
\caption{The heatmaps of $1$ trial of ten robots covering the learned time-varying density distribution using O-RFGP.
}
\vspace{-0.2cm}
\label{TV_ORFGP_sim}
\end{figure*}

\begin{figure*}[t]
\centering
\vspace{-2mm}
\subfigure[$T=10$]{
\begin{minipage}[b]{0.18\textwidth}
\includegraphics[width=1\textwidth]{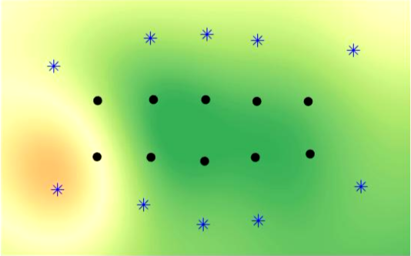}
\end{minipage}
\label{TV_GP_iter10}
}
\subfigure[$T=30$]{
\begin{minipage}[b]{0.18\textwidth}
\includegraphics[width=1\textwidth]{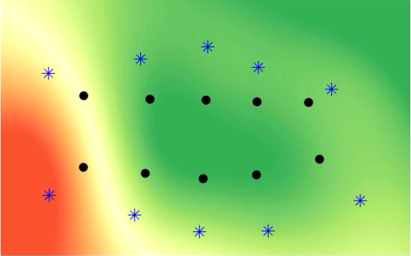}
\end{minipage}
\label{TV_GP_iter30}
}
\subfigure[$T=100$]{
\begin{minipage}[b]{0.18\textwidth}
\includegraphics[width=1\textwidth]{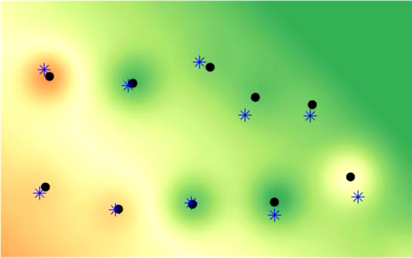}
\end{minipage}
\label{TV_GP_iter100}
}
\subfigure[$T=1200$]{
\begin{minipage}[b]{0.18\textwidth}
\includegraphics[width=1\textwidth]{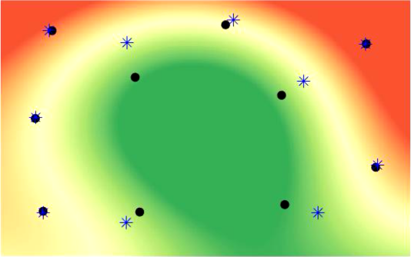}
\end{minipage}
\label{TV_GP_iter1200}
}
\subfigure[$T=2000$]{
\begin{minipage}[b]{0.18\textwidth}
\includegraphics[width=1\textwidth]{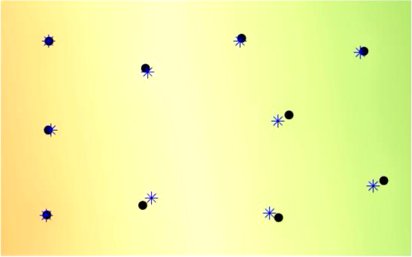}
\end{minipage}
\label{TV_GP_iter2000}
}
\caption{The heatmaps of $1$ trial of ten robots covering the learned time-varying density distribution using GP based on the measurements collected at the current time step
}
\label{TV_GP_sim}
\vspace{-2mm}
\end{figure*}

\begin{figure*}[t]
\centering
\vspace{-1mm}
\subfigure[$T=10$]{
\begin{minipage}[b]{0.18\textwidth}
\includegraphics[width=1\textwidth]{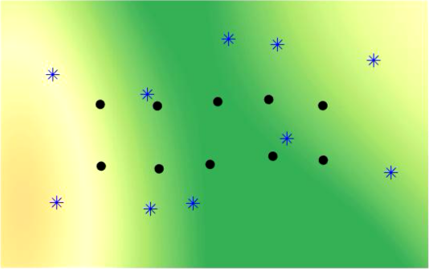}
\end{minipage}
\label{TV_RFGP_iter10}
}
\subfigure[$T=30$]{
\begin{minipage}[b]{0.18\textwidth}
\includegraphics[width=1\textwidth]{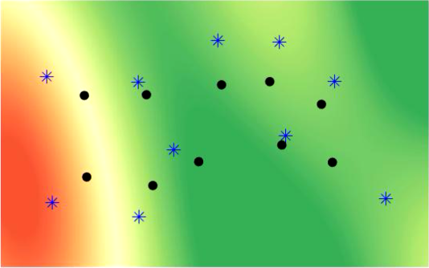}
\end{minipage}
\label{TV_RFGP_iter30}
}
\subfigure[$T=100$]{
\begin{minipage}[b]{0.18\textwidth}
\includegraphics[width=1\textwidth]{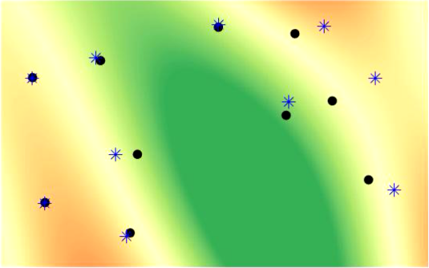}
\end{minipage}
\label{TV_RFGP_iter100}
}
\subfigure[$T=1200$]{
\begin{minipage}[b]{0.18\textwidth}
\includegraphics[width=1\textwidth]{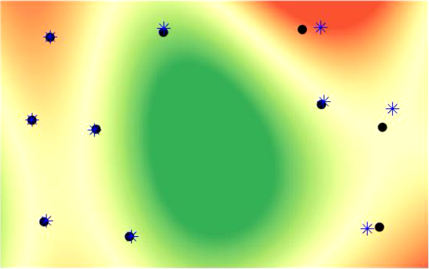}
\end{minipage}
\label{TV_RFGP_iter1200}
}
\subfigure[$T=2000$]{
\begin{minipage}[b]{0.18\textwidth}
\includegraphics[width=1\textwidth]{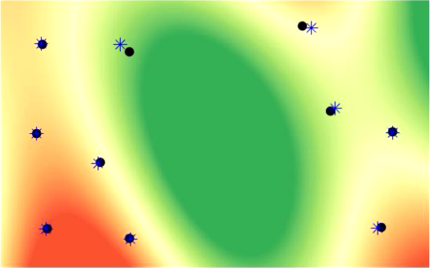}
\end{minipage}
\label{TV_RFGP_iter2000}
}
\caption{The heatmaps of $1$ trial of ten robots covering the learned time-varying density distribution using RFGP based on the measurements collected at the current time step. 
}
\vspace{-0.2cm}
\label{TV_RFGP_sim}
\end{figure*}

\begin{figure}[htb]
\centering
\vspace{-2mm}
\subfigure[]{
\begin{minipage}[b]{0.3\textwidth}
\includegraphics[width=1\textwidth]{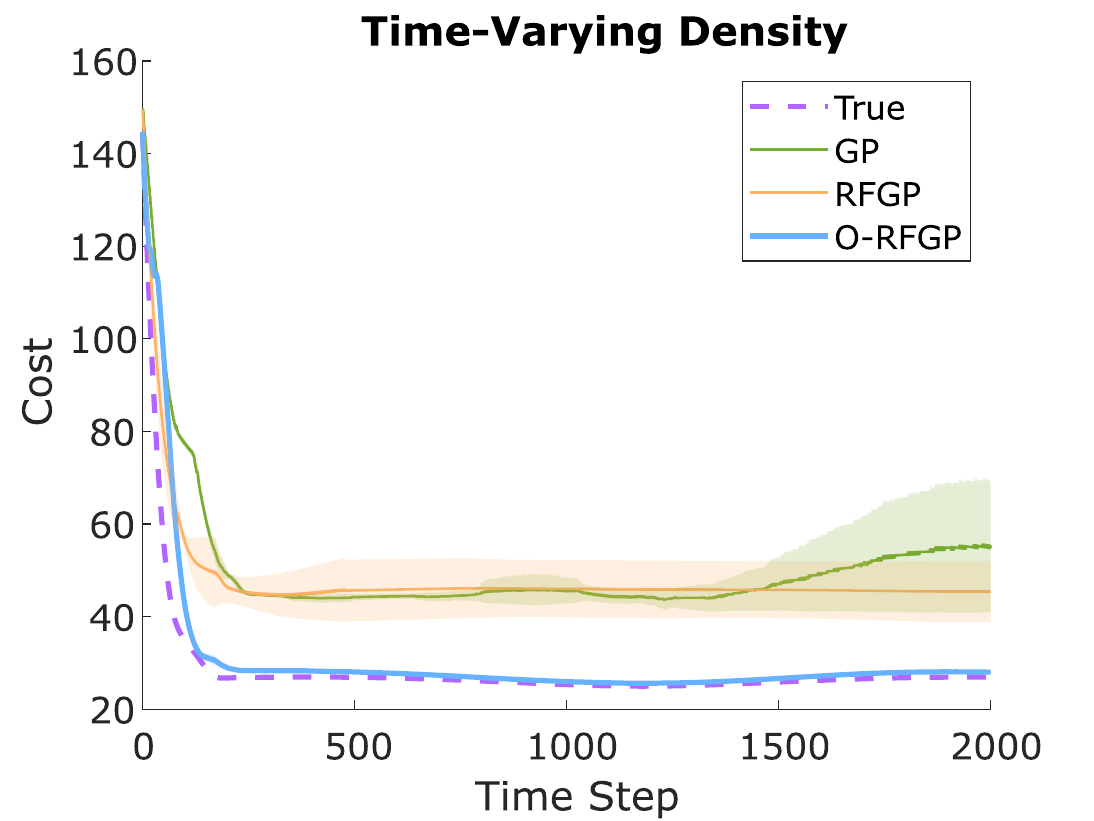}
\end{minipage}
\label{Cost_TV_windowsize1}
}
\subfigure[]{
\begin{minipage}[b]{0.3\textwidth}
\includegraphics[width=1\textwidth]{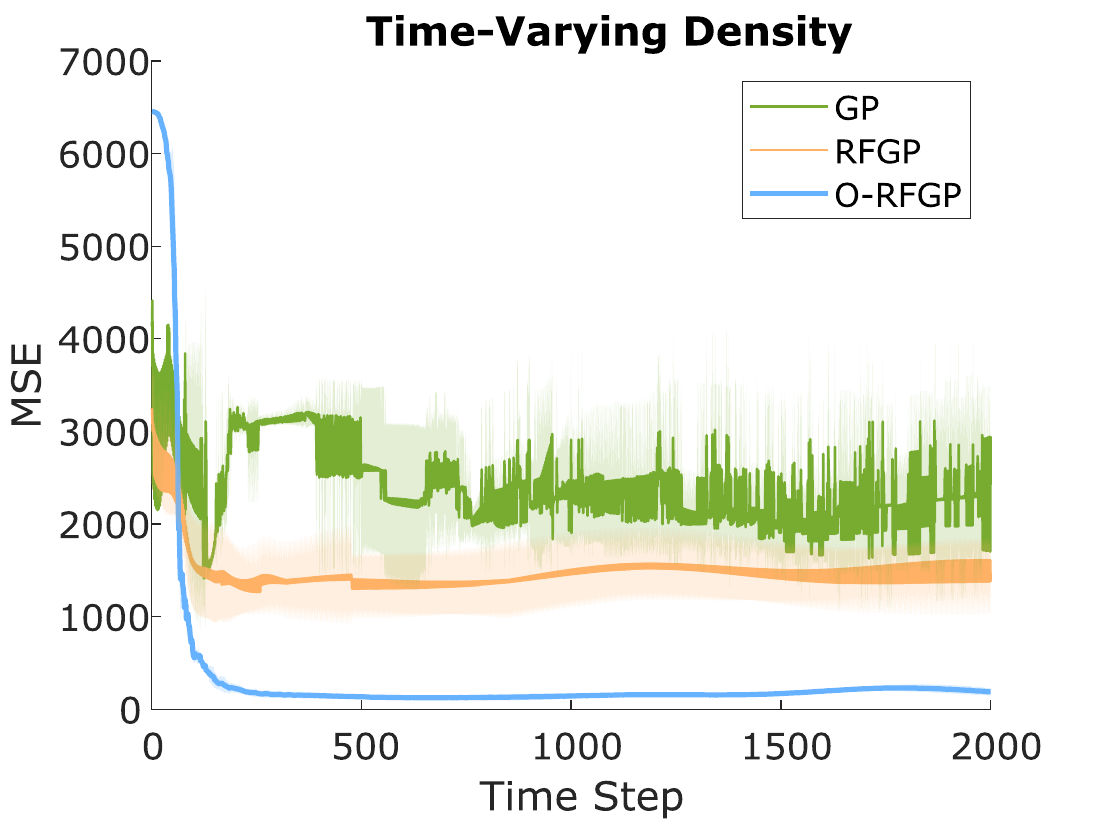}
\end{minipage}
\label{MSE_TV_windowsize1}
}
\caption{The evolution of the locational cost and the MSE between the learned density and the true density with respect to time step. The measurements collected by robots at the current time step are leveraged when using the GP, RFGP, and O-RFGP. 
}
\label{Data_TV_windowsize1}
\vspace{-0.3cm}
\end{figure}

\subsection{Simulations} 
\label{sec:simulations}
In this subsection, we conduct simulations to evaluate and compare the performance of the proposed algorithm on time-invariant and time-varying density functions separately. The detailed experimental settings are presented in Appendix~\ref{sec:experimentalsettings}. A summary of how to use each learning method to estimate the (time-invariant or time-varying) density is as follows: \textbf{ GP}: compute $\mu_t(\cdot)$ and $\sigma_t(\cdot)$ as equations \eqref{eq:GP_mean} and \eqref{eq:GP_var};
\textbf{RFGP:} compute $\mu_t(\cdot)$ and $\sigma_t(\cdot)$ as equations \eqref{eq:RFGP}; \textbf{O-RFGP:} compute $\mu_t(\cdot)$ and $\sigma_t(\cdot)$ as equations \eqref{eq:O-RFGP} by iteratively updating $\mathbf{\theta}$ and $\mathbf{\Sigma}$ as equations \eqref{eq:O-RFGP_theta} and \eqref{eq:O-RFGP_Sigma}. We use the mean squared error (MSE) between $f_t$ and $\hat{f}_t$ to evaluate the learning quality, and compute the locational cost using $f_t$ as equation \eqref{location_cost_objective} to evaluate the coverage quality.
{More experimental results on the sensitivity analysis of hyperparameters $D$ and $\beta_t$, and the computational efficiency comparison of GP, RFGP, and O-RFGP are presented in Appendix~\ref{sec:additionalexp}.}




\subsubsection{Time-invariant Density} \label{exp:time-invariant}



The measurements collected only at the current time step are leveraged when using O-RFGP, while all measurements collected from the beginning
up to the current time step are leveraged when using the GP
and RFGP for learning the time-invariant density. The simulation results are shown in Fig.\ref{CC_TI_GP_RFGP_ORFGP} and Fig.\ref{Data_TI}.
The final configurations of one trail, at time step $500$, of the robots covering the learned time-invariant density using the GP, RFGP and O-RFGP are shown in Fig.\ref{GP_TI}, Fig.\ref{RFGP_TI} and Fig.\ref{ORFGP_TI}, respectively. 
{Fig.\ref{Cost_TI} shows that the coverage qualities achieved by the GP, RFGP, and O-RFGP are almost the same as the coverage quality over the ground truth density. Fig.\ref{MSE_TI} and Fig.\ref{CC_TI_GP_RFGP_ORFGP} indicate that O-RFGP and RFGP have higher MSE compared with GP.}
These results are consistent with the regret analysis in Appendix~\ref{sec:regret} since RFGP and O-RFGP have larger $\gamma_T$. 
It is worth noting that although all three learning methods achieve similar coverage quality in the end, O-RFGP and RFGP do so with significantly lower computational cost than GP, as seen in Table~\ref{table:computation_complex_memory_cost} and Appendix~\ref{sec:computationalcomplexity}.
Unlike the RFGP and O-RFGP, the GP has cubic computational complexity in the number of training samples. This implies the limitation in applying the GP in real-time robotic implementations as discussed in Section~\ref{Section:Introduction}.

\subsubsection{Time-varying Density} \label{exp:time-varying}
We conduct simulations using the measurements that robots collect only at the current time step to avoid stale data for time-varying density learning. The heatmaps of one trial of ten robots covering the learned density using O-RFGP, GP, and RFGP are shown in Fig.\ref{TV_ORFGP_sim}, Fig.\ref{TV_GP_sim}, and Fig.\ref{TV_RFGP_sim}, respectively. 
Fig.\ref{Data_TV_windowsize1} shows that the coverage quality and MSE for density learning resulting from GP and RFGP are apparently higher than those from O-RFGP. 
In Fig.\ref{MSE_TV_windowsize1}, the MSE values of RFGP and GP are not only 
higher than that of O-RFGP but also exhibit high-frequency fluctuations. These fluctuations in MSE lead to higher and less stable coverage costs for GP and RFGP compared to O-RFGP, as shown in Fig.\ref{Cost_TV_windowsize1}.
Compared with Fig.\ref{TV_True}, the heatmap results in Fig.\ref{TV_GP_sim} and Fig.\ref{TV_RFGP_sim} also indicate that both GP and RFGP do not provide accurate estimation results in the time-varying setting. In contrast, Fig.\ref{TV_ORFGP_sim} shows that O-RFGP is the only method capable of accurately estimating the true density over time.
Since GP and RFGP are designed for learning the time-invariant density using all historical data, as discussed in Section \ref{Sec:ORC}, they are not supposed to adapt to the evolution of the time-varying density in a straightforward manner.
In contrast, O-RFGP can adapt to the evolution of the time-varying density, achieving almost the same coverage quality as the ground truth density across time, along with accurate estimation. Since O-RFGP iteratively updates the learned density using current measurements as the density varies over time, it is not influenced by out-of-date measurements and able to capture the dynamic property of the time-varying density, as discussed in Section~\ref{sec:O-RFGP}.

In summary, O-RFGP not only benefits from lightweight computation and memory, but also employs the incremental update structure that only relies on the most recent collected data.
These two properties enable real-time learning and adaptation, and the simulation results further
suggest the potential
of O-RFGP in handling time-varying density for real-world robotic applications.



\begin{figure*}[t]
\centering
\vspace{-2mm}
\subfigure[$T=10$]{
\begin{minipage}[b]{0.18\textwidth}
\includegraphics[width=1\textwidth]{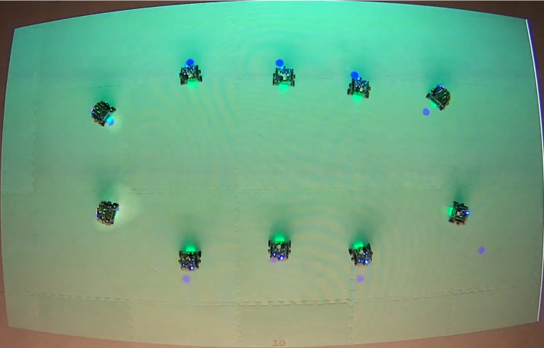}
\end{minipage}
\label{TV_ORFGP_real_iter10}
}
\subfigure[$T=30$]{
\begin{minipage}[b]{0.18\textwidth}
\includegraphics[width=1\textwidth]{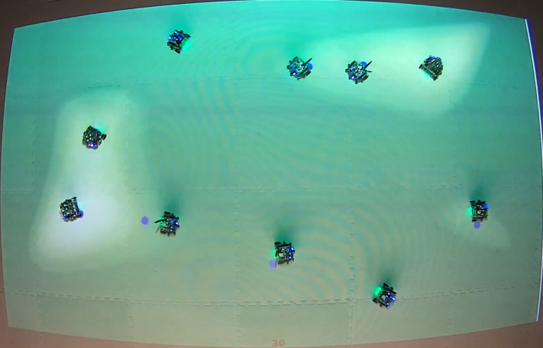}
\end{minipage}
\label{TV_ORFGP_real_iter30}
}
\subfigure[$T=100$]{
\begin{minipage}[b]{0.18\textwidth}
\includegraphics[width=1\textwidth]{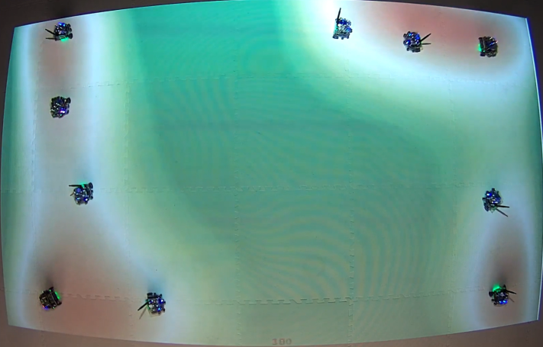}
\end{minipage}
\label{TV_ORFGP_real_iter100}
}
\subfigure[$T=1200$]{
\begin{minipage}[b]{0.18\textwidth}
\includegraphics[width=1\textwidth]{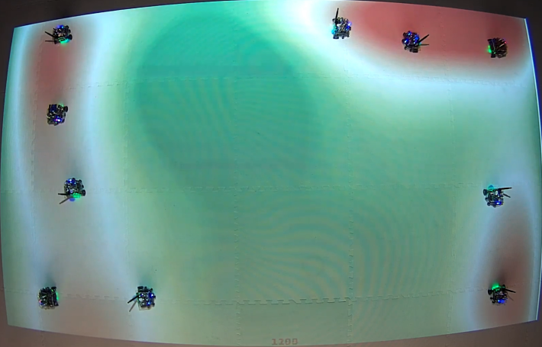}
\end{minipage}
\label{TV_ORFGP_real_iter1200}
}
\subfigure[$T=2000$]{
\begin{minipage}[b]{0.18\textwidth}
\includegraphics[width=1\textwidth]{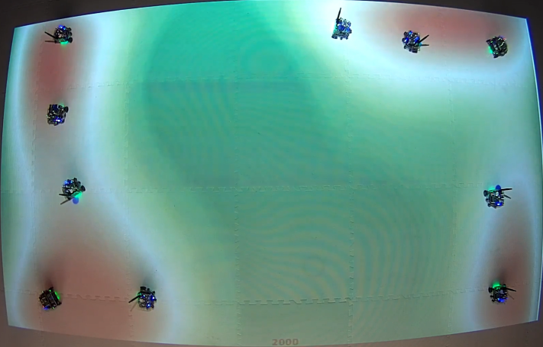}
\end{minipage}
\label{TV_ORFGP_real_iter2000}
}
\caption{The screenshots of ten differential-drive wheeled robots covering the learned time-varying density distribution using O-RFGP. The blue dots represent the centers of mass of their corresponding Voronoi cells. 
}
\label{TV_ORFGP_real}
\vspace{-0.3cm}
\end{figure*}

\begin{figure}[htb]
\centering
\vspace{-2mm}
\subfigure[]{
\begin{minipage}[b]{0.3\textwidth}
\includegraphics[width=1\textwidth]{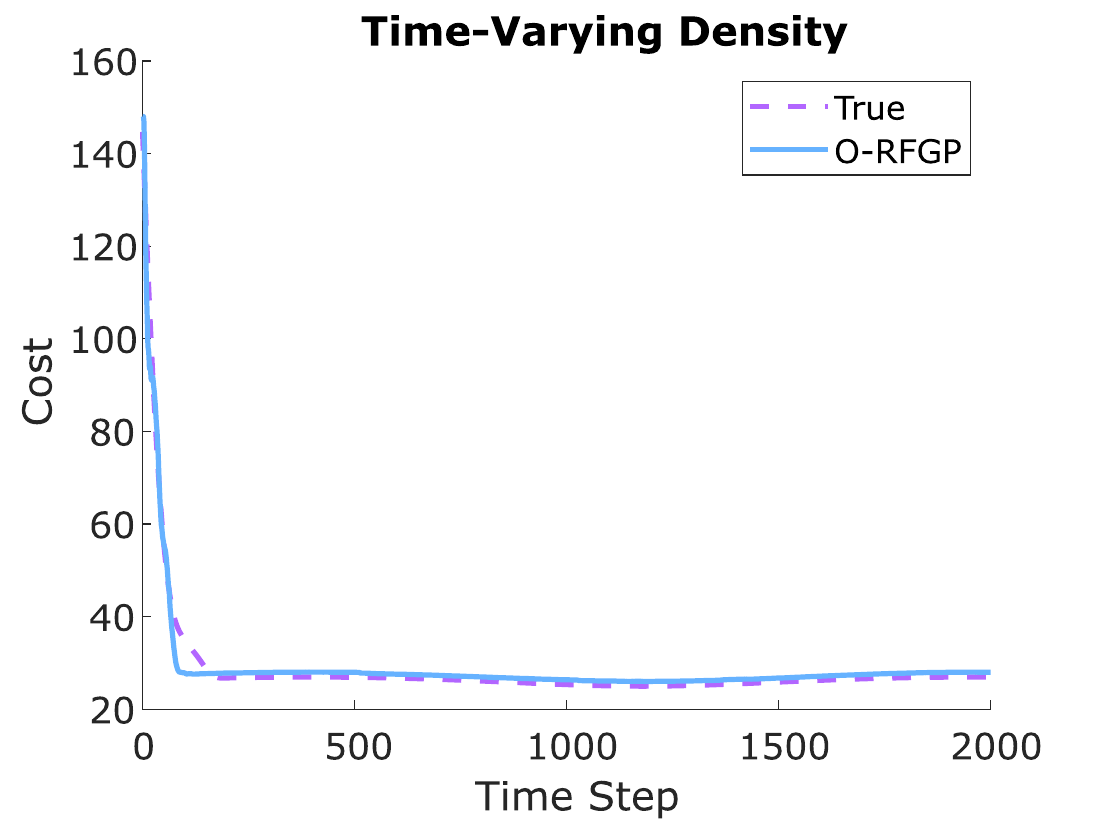}
\end{minipage}
\label{Cost_TV_real}
}
\subfigure[]{
\begin{minipage}[b]{0.3\textwidth}
\includegraphics[width=1\textwidth]{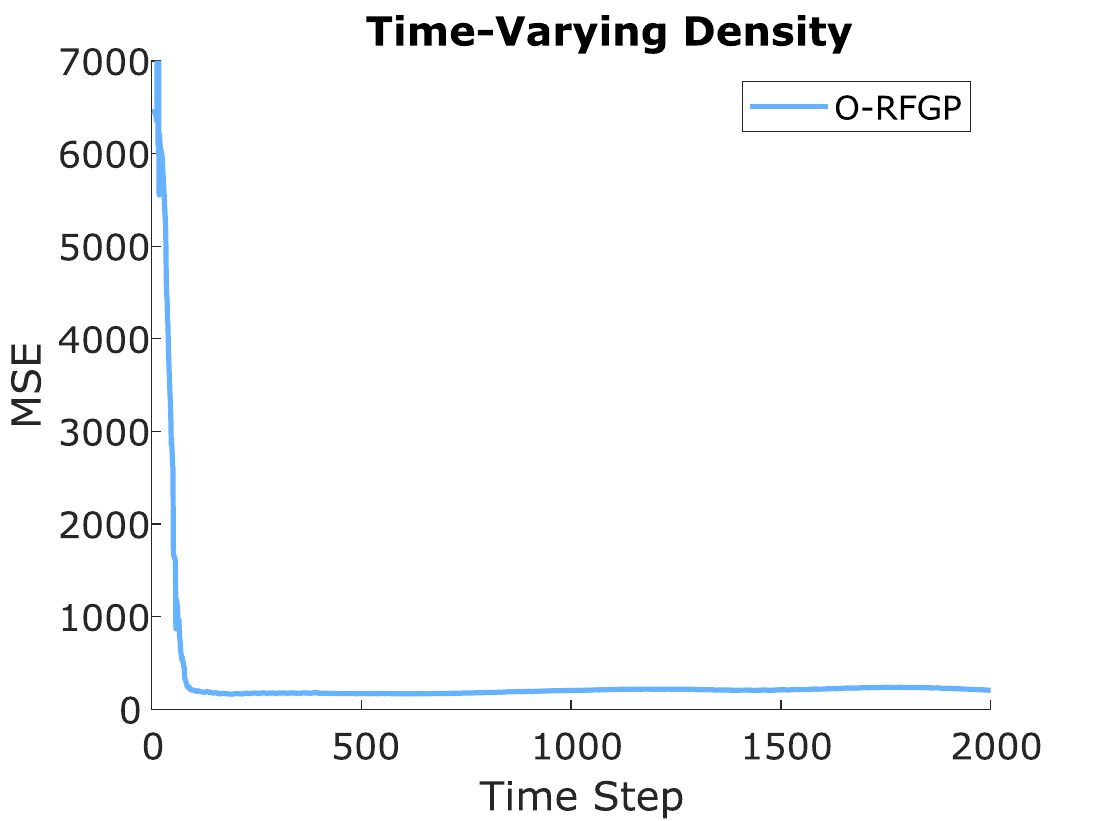}
\end{minipage}
\label{MSE_TV_real}
}
\caption{The evolution of the locational cost and the MSE between the learned density in the physical experiment (corresponding to Fig.\ref{TV_ORFGP_real}) and the true density in the simulation (corresponding to Fig.\ref{TV_True}). The measurements collected by robots at the current time step are leveraged to learn the time-varying density using O-RFGP.}
\vspace{-3mm}
\label{Data_TV_real}
\end{figure}






\subsection{Hardware-in-the-Loop Experiment} \label{sec:physical experiment}

To highlight the performance and effectiveness of ORC in learning time-varying density on physical robotic systems in real-time,
we implement it
on ten differential-drive wheeled robots for simultaneous learning and covering an unknown time-varying density.
The size of the physical testbed~\citep{pickem2017robotarium} is the same as the domain of interest $\mathcal{D}$ in numerical simulations. The location of Robot $i$, ${x}^i$, corresponds to the position of a point near the wheel axle of the robot. Single integrator dynamics of ${x}^i$ is mapped by the near-identity diffeomorphism~\citep{olfati2002near} to the unicycle dynamics of the differential-drive wheeled robot.
The experimental results are shown in Fig.\ref{TV_ORFGP_real} and Fig.\ref{Data_TV_real}. 
As shown in Fig.\ref{Cost_TV_real}, the evolution of the locational cost for the time-varying density learned by O-RFGP closely matches that of the ground truth, indicating comparable coverage quality. In Fig.\ref{MSE_TV_real}, the MSE between the learned and true densities also converges, confirming the effectiveness of O-RFGP in density learning.
Comparing Fig.\ref{TV_True}, Fig.\ref{TV_ORFGP_sim}, Fig.\ref{Data_TV_windowsize1}, Fig.\ref{TV_ORFGP_real}, and Fig.\ref{Data_TV_real}, O-RFGP performs consistently across both simulation and physical experiments. It is demonstrated that Algorithm \ref{alg:SLAC} enables multi-robot systems to efficiently learn and cover unknown, time-varying spatial fields in real time.

\section{Conclusion}
This paper proposes ORC (Algorithm \ref{alg:SLAC}), a framework that integrates O-RFGP and Voronoi-based coverage control to enable a team of robots to learn an unknown density function and achieve efficient spatial coverage. 
To the best of our knowledge, we provide, for the first time, a regret-based theoretical analysis that guarantees the performance of online learning–aided coverage control in time-invariant settings. 
The effectiveness of ORC is validated through simulations, and its strong potential in handling time-varying density fields is further demonstrated through additional simulations and a physical experiment, confirming its capability for real-time learning and coverage.


\bibliography{main}
\begin{appendices}
    \appendix
\section{Regret Analysis} \label{sec:regret}
In this section, we provide the regret analysis in the time-invariant setting, i.e. the density function $f$, as in \eqref{eq:true_density}, does not vary with respect to time.
Remind that the regret, $r_t$, and the cumulative regret, $R_T$, are defined in equation \eqref{cum_regret}:
\begin{align*} 
    r_t &= \ell(f,\mathbf{x}_t) - \min_{\mathbf{x}}\ell(f,\mathbf{x}) \geq 0, \\
    R_T &\triangleq \sum_{t=1}^T r_t. \nonumber
\end{align*}
The \textit{asymptotic no-regret} property, i.e., $\lim_{T\rightarrow \infty} \frac{R_T}{T}=0$, is based on the sub-linear growth of $R_T$ in $T$ achieved by \eqref{eq:sequence}, which will be proved in the following content.
A small cumulative regret $R_T$ indicates that the achieved coverage cost $\ell(f,\mathbf{x}_t)$ remains close to the optimal cost, meaning that the robots attain coverage performance comparable to the ideal case with full knowledge of the true density function.
In Algorithm \ref{alg:SLAC}, the non-decreasing sequence $\{ \beta_t \}_{t=1}^T$ is required to achieve sub-linear regret, especially with the UCB sampling strategy, e.g., \citet{calandriello2019gaussian,srinivas2012information,PCC_cov}. Specifically we let
\begin{equation}
\label{eq:sequence}
\beta_t \triangleq 2\log\left( |\mathcal{D}| \frac{t^2 \pi^2}{6\delta}\right)
\end{equation}
with $\delta \in (0,1]$. The \textit{asymptotic no-regret} property, i.e., $\lim_{T\rightarrow \infty} \frac{R_T}{T}=0$, is based on the sub-linear growth of $R_T$ in $T$ achieved by \eqref{eq:sequence}. In Section \ref{sec:experiments}, we set $\beta_t$ as a constant value and observe that sublinear regret still remained.
To proceed, we will need the following assumptions:
\begin{assumption}
    The optimal hyperparameters of $f(\cdot)$ exit, denoted as $\rho^*$ and $\tau^*$ in equation~\eqref{eq:true_density}.\label{assump_1}
\end{assumption}
\begin{assumption}
    The domain $\mathcal{D}$ is discretized. \label{assump_2}
\end{assumption}
\begin{assumption}
    The control law  solves 
    the optimization problem $\mathbf{x}_t = \argmin_{\mathbf{x} }
    \ell(\hat{f}_t,\mathbf{x})
    $. 
    \label{assump_3}
\end{assumption}
Assumption \ref{assump_1} simplifies the proof: {it implies that the function is sampled from a known GP and has a low reproducing kernel Hilbert space (RKHS) norm \citep{srinivas2009gaussian}.} 
Note that $f(\cdot)$ represents the true density function, so $\rho^*$ and $\tau^*$ are considered as hyperparameters of the true density and thus not accessible to the robots. Assumption \ref{assump_2} can also be further relaxed since the results can be generalized to a compact domain, and $\mathcal{D}$ is discretized either in the simulations or the physical experiment of this paper.
Assumption~\ref{assump_3} is a theoretical idealization used for the regret analysis. However, the Voronoi-based coverage control law \eqref{eq:tv_ctrl} typically achieves local optima without guaranteeing global optimality, e.g., \citet{cortes2004coverage,luo2019distributed}. Nevertheless, extensive experiments in Section~\ref{sec:experiments} empirically suggest that a local optimum (a CVT) reached by Algorithm~\ref{alg:SLAC} performs well in practice and may, in some cases, reach the global optimum. Importantly, our proposed framework provides the flexibility that alternative control laws can be directly incorporated without changing other components of Algorithm~\ref{alg:SLAC}. Should future theoretical work prove that certain control laws attain global optimality, our analysis in this paper remains valid.



To establish the static regret bound {using different density learning methods},
we first introduce the following lemma.

\begin{lemma} \label{lemma1}
    Pick $\delta \in (0,1]$, and set $\beta_t = 2\log\left( |\mathcal{D}| \frac{\pi_t}{\delta}\right)$, where $\sum_{t \geq 1}\pi_t^{-1}=1, \pi_t >0$. Then,
    \begin{align*}
        \left| f({x}) - \mu_{t-1}({x}) \right| \leq \sqrt{\beta_t} \sigma_{t-1}({x})
    \end{align*}
    holds for any ${x}\in \mathcal{D}$ and any $t\geq 1$ with probability at least $1-\delta$.
\end{lemma}
\begin{proof}
For GP, see Lemma 5.6 of \citet{srinivas2012information}, while for RFGP and O-RFGP, see Appendices \ref{appendix_lemma1} and \ref{converge_ORFGP}.
\end{proof}

We assume that the high probability event of Lemma~\ref{lemma1} holds. Next, we will work towards the upper bound of accumulated regret through the following theorem.

\begin{theorem} \label{theorem1}
    Let $\delta \in (0,1]$ and $\beta_t = 2\log\left( |\mathcal{D}| \frac{t^2 \pi^2}{6\delta}\right)$. If the density function $f(\cdot)$ is the realization of a Gaussian Process with mean $\mu(\cdot)$ and covariance function $k(\cdot,\cdot)\leq 1$, then running Algorithm \ref{alg:SLAC} ensures that 
    \begin{align*}
        \forall T\geq1, ~~ R_t &\leq \sqrt{c_1 (n+1)T \beta_T \gamma_T}
    \end{align*}
    holds with probability at least $1-\delta$, where $c_1 \triangleq 8 |D|^2 \max_{{x},{x}'} \frac{\lVert{x}-{x}'\rVert^4_2}{\log(1+\sigma^{-2})}$. For kernels {such that $\gamma_T = \mathcal{O}(T)$}, Algorithm~\ref{alg:SLAC} has \textit{asymptotic no-regret}.
\end{theorem}
\begin{proof}
See Appendix \ref{appendix_theorem1}.
\end{proof}

Note that the estimation step can also be implemented using either GP or RFGP. As shown in \citet{srinivas2012information}, $\gamma_T$ scales sub-linearly with $T$ for GP with many kernels, e.g., $\gamma_T=\mathcal{O}(\log T)$ with RBF kernel. 
However, $\gamma_T$ for RFGP and O-RFGP usually cannot achieve information as well as GP since RFGP usually suffers from variance starvation \citep{calandriello2019gaussian, wang2018batched}, e.g., $\gamma_T=\mathcal{O}(\sqrt{T})$ with RBF kernel, details are shown in Appendix~\ref{inf_gain_RFGP}. Nevertheless, RFGP and O-RFGP still ensure that $\gamma_T$ scales sub-linearly, i.e., $\gamma_T = \mathcal{O}\left(2D\log\left(\frac{T}{D}\right) \right)$, with details in Appendix~\ref{inf_gain_RFGP}.

As discussed in \citet{lu2020ensemble,shen2019random}, O-RFGP could potentially handle time-varying density, $f_t$. 
We leave the regret analysis for the time-varying setting as future work. Nonetheless,
Algorithm~\ref{alg:SLAC} is evaluated with a time-varying density function in Section \ref{exp:time-varying} and Section \ref{sec:physical experiment}, where the results are {encouraging in the sense that O-RFGP
learns a surrogate $\hat{f}_t$, 
and by running the control law \eqref{eq:tv_ctrl} with $\hat{f}_t$, the team of robots can be guided to locations that efficiently cover the domain of interest, i.e., a time-varying CVT}.

\section{Proof of Lemma \ref{lemma1}} \label{appendix_lemma1}

In this section, we prove that Lemma \ref{lemma1} holds when the RFGP approach computes $\mu$:
\begin{align*}
        \left| f({x}) - \mu_{t-1}({x}) \right| \leq \sqrt{\beta_t} \sigma_{t-1}({x})
\end{align*}
holds for any ${x}\in \mathcal{D}$ and any $t\geq 1$ with probability at least $1-\delta$ while picking $\delta \in (0,1]$, and set $\beta_t = 2\log\left( |\mathcal{D}| \frac{\pi_t}{\delta}\right)$, where $\sum_{t \geq 1}\pi_t^{-1}=1, \pi_t >0$.

\textit{Proof:}
Since we can have $f({x}) \sim \mathcal{N}(\mu_{t-1}({x}), \sigma^2_{t-1}({x}))$.
Let $r\sim \mathcal{N}(0,1)$ and $c_5>0$, we can have
\begin{align*}
    P(r>c) &= e^{\frac{-c_5^2}{2}} \frac{1}{\sqrt{2 \pi}} \int e^{-\frac{(r-c_5)^2}{2-(r-c_5)}}
    \leq e^{\frac{-c_5^2}{2}} P(r>0) = \frac{1}{2} e^{\frac{-c^2}{2}}.
\end{align*}
Let $r=\frac{f({x}) - \mu_{t-1}}{\sigma_{t-1}({x}) }$ and $c_5=\sqrt{\beta_t}$, then
\begin{align*}
    P( \left| f({x}) - \mu_{t-1} \right|  > \sqrt{\beta_t} \sigma_{t-1}({x})) \leq e^{-\frac{\beta_t}{2}}.
\end{align*}
Since the domain $\mathcal{D}$ is discrete per \textit{Assumption \ref{assump_2}}, applying the union bound,
\begin{align*}
    P\left( \left| f({x}) - \mu_{t-1} \right|  > \sqrt{\beta_t} \sigma_{t-1}({x}), ~ \forall {x}\in \mathcal{D} \right) \leq |\mathcal{D}| e^{-\frac{\beta_t}{2}},
\end{align*}
which can be rewritten as
\begin{align}
    P\left( \left| f({x}) - \mu_{t-1} \right|  \leq \sqrt{\beta_t} \sigma_{t-1}({x}), ~ \forall {x}\in \mathcal{D} \right) \geq 1 - |\mathcal{D}| e^{-\frac{\beta_t}{2}}.
\end{align}
Choosing $|D| e^{-\frac{\beta_t}{2}} = \frac{\delta}{\pi_t}$, we can have $\beta_t = 2 \log\left(|\mathcal{D}|\frac{\pi_t}{\delta} \right)$:
\begin{align*}
    P\left( \left| f({x}) - \mu_{t-1} \right|  \leq \sqrt{\beta_t} \sigma_{t-1}({x}), ~ \forall {x}\in \mathcal{D} \right) \geq 1- \delta.
\end{align*}

Proof of Lemma \ref{lemma1} is completed.

\section{Proof of Theorem \ref{theorem1}} \label{appendix_theorem1}
In the proof of Theorem \ref{theorem1}, we assume the high-probability event of Lemma \ref{lemma1} holds. Firstly, we define
\begin{align*}
    \ell^* &\triangleq \min_{\mathbf{x}} \ell(f, \mathbf{x}) \\
    \mathbf{x}^* &\triangleq \argmin_{\mathbf{x}}\ell(f, \mathbf{x})
\end{align*}
Based on \textit{Assumption \ref{assump_3}}, $\mathbf{x}_t = \argmin_{\mathbf{x}} \ell(\hat{f}_t, \mathbf{x})$, we have
\begin{align} \label{cost inequality}
    \ell(\hat{f}_t, \mathbf{x}_t) \leq \ell(\hat{f}_t, \mathbf{x}^*) \leq \ell^*
\end{align}
where the last inequality follows since the event in Lemma \ref{lemma1} holds:
\begin{align*}
     \mu_{t-1}({x}) &- \sqrt{\beta_t} \sigma_{t-1}({x}) \leq f({x}) \leq \mu_{t-1}({x}) + \sqrt{\beta_t}, \\
     \hat{f}_t({x}) & = \mu_{t-1}({x}) - \sqrt{\beta_t} \sigma_{t-1}({x}), \\
     \hat{f}_t({x}) &\leq f({x}) ~~~ \forall {x}\in \mathcal{D},\\
     \ell(\hat{f}_t, \mathbf{x}^*) &\leq \ell^*.
\end{align*}
Then, the regret at step $t$ is
\begin{align}
    r_t &\triangleq  \ell(f, \mathbf{x}_t) - \ell^* \stackrel{(a)}{\leq} \ell(f, \mathbf{x}_t) - \ell(\hat{f}_t, \mathbf{x}_t) \nonumber\\
    & \stackrel{(b)}{=} \sum_i \sum_{{q}\in \mathcal{V}^i_t} \lVert{q} - {x}^i_t\rVert_2^2 (f({q}) - \hat{f}_t({q})) \nonumber \\
    & \stackrel{(c)}{\leq} 2 \sqrt{\beta_t} \left( \sum_i \sum_{{q}\in \mathcal{V}^i_t} \lVert{q} - {x}^i_t\rVert_2^2 \right) \max_{{q}\in \mathcal{D}} \sigma_{t-1}({q}) \nonumber \\
    &\leq 2 c_0 \sqrt{\beta_t} \max_{{q}\in \mathcal{D}} \sigma_{t-1}({q}) \triangleq 2 c_0 \sqrt{\beta_t} \sigma_{t-1}(\Tilde{{x}}_t),
\end{align}
where $(a)$ follows from \eqref{cost inequality}, $(b)$ follows by the definition of locational cost, $(c)$ holds because $f({q}) \leq \mu_{t-1}({q}) + \sqrt{\beta_t} \sigma_{t-1}({q})$ and $\hat{f}_t({q}) = \mu_{t-1}({q}) - \sqrt{\beta_t} \sigma_{t-1}({q})$, and $c_0$ is defined as
\begin{align*}
    c_0 = |\mathcal{D}| \max_{{x},{x}' \in \mathcal{D}} \lVert {x}-{x}'\rVert_2^2 .
\end{align*}
Note that $c_0$ only depends on the domain $\mathcal{D}$ and not on the number of robots $n$.

Since $\beta_t$ is non-decreasing with respect to $t$, we have
\begin{align}
    (r_t)^2 &\leq 4 c_0^2 \beta_t \left( \sigma_{t-1}(\Tilde{{x}}_t)\right)^2 \\
    & \leq c_1 \beta_T \frac{1}{2} \log \left( 1+ \left( \frac{\sigma_{t-1}(\Tilde{{x}}_t)}{\sigma} \right)^2 \right),
\end{align}
by setting $c_1 = \frac{8 c_0^2}{\log(1+\sigma^{-2})}$. 
Then, we have
\begin{align} \label{upperbound sum rt^2}
    \sum_{t=1}^T (r_t)^2 \leq c_1 \beta_T \sum_{t=1}^T \log \left( 1+ \left( \frac{\sigma_{t-1}(\Tilde{{x}}_t)}{\sigma} \right)^2 \right).
\end{align}
According to \eqref{upperbound sum rt^2}, we will derive the information gain which upper bounds the last term, i.e., $$\sum_{t=1}^T \log \left( 1+ \left( \frac{\sigma_{t-1}(\Tilde{\mathbf{x}}_t)}{\sigma} \right)^2 \right)$$. For $t \in \{1,\cdots,T\}$, we define $f(\mathbf{x}_{:t}) \in \mathbb{R}_{nt}$ as
\begin{align*}
    f(\mathbf{x}_{:t}) \triangleq \left[ f(\mathbf{x}^1_1) \cdots f(\mathbf{x}^n_1) \cdots f(\mathbf{x}^1_t) \cdots f(\mathbf{x}^n_t)) \right]^\top.
\end{align*}
Recall that for every robot $i \in \{1,\cdots,n\}$ at each $t$, we have
\begin{align*}
    y^i_t = f({x}^i_t) +\epsilon,
\end{align*}
where all noises are $\mathcal{N}(0,\sigma^2)$ and i.i.d. Then,
\begin{align}
    I(f(\mathbf{x}_{:T}),f(\Tilde{\mathbf{x}}_{:T} ); \mathbf{y}_{:T}, \Tilde{\mathbf{y}}_{:T}) &= h(\mathbf{y}_{:T}, \Tilde{\mathbf{y}}_{:T})-h(\mathbf{y}_{:T}, \Tilde{\mathbf{y}}_{:T} | f(\mathbf{x}_{:T}),f(\Tilde{\mathbf{x}}_{:T} )) \nonumber\\
    &\stackrel{(a)}{=} h(\mathbf{y}_{:T}, \Tilde{\mathbf{y}}_{:T}) - \frac{(n+1)T}{2} \log(2\pi e \sigma^2),
\end{align}
where $(a)$ follows by
\begin{align*}
    h(\mathbf{y}_{:T}, \Tilde{\mathbf{y}}_{:T} | f(\mathbf{x}_{:T}),f(\Tilde{\mathbf{x}}_{:T} )) 
    &= - \sum_{y\in \mathbf{y}_{:T}, \Tilde{\mathbf{y}}_{:T}} \log p(y)\\
    & = - \sum_{y\in \mathbf{y}_{:T}, \Tilde{\mathbf{y}}_{:T}} \log \left[ \frac{1}{\sqrt{2\pi \sigma^2} e^{\frac{(y-f({x}))^2}{2\sigma^2}}}\right]\\
    & = - \sum_{y\in \mathbf{y}_{:T}, \Tilde{\mathbf{y}}_{:T}} \left( -\frac{1}{2}\log(2\pi \sigma^2)-0 \right)\\
    & = \frac{(n+1)T}{2} \log(2\pi \sigma^2).
\end{align*}
By the chain rule of differential entropy, we have
\begin{align}
    h(\mathbf{y}_{:T}, \Tilde{\mathbf{y}}_{:T}) 
    &= \sum_{t=1}^T h(\mathbf{y}_t, \Tilde{y}_t; \mathbf{y}_{:t-1}, \Tilde{\mathbf{y}}_{:t-1}) \nonumber \\
    & = \sum_{t=1}^T h(\mathbf{y}_t; \mathbf{y}_{:t-1}, \Tilde{\mathbf{y}}_{:t-1}) + h(\Tilde{y}_t; \mathbf{y}_{:t-1}, \Tilde{\mathbf{y}}_{:t-1})
\end{align}

According to ORC algorithm, the sequence of sensing points $\mathbf{x}_{:t}$ and $\Tilde{\mathbf{x}}_{:t}$ is a \textit{deterministic} function of $ \mathbf{y}_{:t-1}$ and $ \Tilde{\mathbf{y}}_{:t-1}$. 
Since $\Tilde{y}_t$ and $\mathbf{y}_t$ conditioned on $\mathbf{x}_{:t}, \Tilde{\mathbf{x}}_{:t}, \mathbf{y}_{:t-1}, \Tilde{\mathbf{y}}_{:t-1}$ that are distributed according to $\mathcal{N}(\mu_{t-1}(\Tilde{{x}}_t),\sigma^2+ \sigma_{t-1}(\Tilde{{x}}_t)^2)$ and $\mathcal{N}(\Tilde{\mathbf{\mu}}_t,K_t(\mathbf{x}_{t}))$ with $\Tilde{\mathbf{\mu}}_t \in \mathbb{R}^{n}$ and $K_t(\mathbf{x}_{t}) \in \mathbb{R}^{n\times n}$, we have
\begin{align}
    h(\Tilde{y}_t; \mathbf{y}_{:t-1}, \Tilde{\mathbf{y}}_{:t-1}) & = h(\Tilde{y}_t; \mathbf{x}_{:t}, \Tilde{\mathbf{x}}_{:t}, \mathbf{y}_{:t-1}, \Tilde{\mathbf{y}}_{:t-1}) \nonumber\\
    &= \frac{1}{2} \log \left( 2\pi e (\sigma^2 + \sigma_{t-1}(\Tilde{{x}}_t)^2) \right), \label{h_tilde_yt}\\
    h(\mathbf{y}_t; \mathbf{y}_{:t-1}, \Tilde{\mathbf{y}}_{:t-1}) &= h(\mathbf{y}_t; \mathbf{x}_{:t}, \Tilde{\mathbf{x}}_{:t}, \mathbf{y}_{:t-1}, \Tilde{\mathbf{y}}_{:t-1}) \nonumber\\
    &= \frac{1}{2} \log \left| 2\pi e (\sigma^2 \mathbf{I}+ K_t(\mathbf{x}_{t}) ) \right|. \label{h_yt} 
\end{align}
Combining \eqref{h_tilde_yt} and \eqref{h_yt} together, we obtain
\begin{align}
    I(f(\mathbf{x}_{:T}),f(\Tilde{\mathbf{x}}_{:T} ); \mathbf{y}_{:T}, \Tilde{\mathbf{y}}_{:T}) 
    &~~~ = \sum_{t-1}^T \left[ \frac{1}{2} \log \left| 2\pi e (\sigma^2 \mathbf{I}+\mathbf{\Sigma}_t) \right| + \frac{1}{2} \log \left( 2\pi e (\sigma^2 + \sigma^2_{t-1}(\Tilde{{x}}_t)) \right) \right] \nonumber\\
    &~~~ - \frac{(n+1)T}{2} \log(2\pi e \sigma^2), \nonumber\\
    &~~~ = \sum_{t-1}^T \left[ \frac{1}{2} \log\left| ( \mathbf{I}+\frac{1}{\sigma^2}\mathbf{K}_t) \right| + \frac{1}{2}\log\left( 1 + \frac{\sigma^2_{t-1}(\Tilde{\mathbf{x}}_t)}{\sigma^2} \right) \right] \nonumber\\
    &~~~ \geq \sum_{t-1}^T \frac{1}{2} \log  \left( 1+ \left( \frac{\sigma_{t-1}(\Tilde{{x}}_t)}{\sigma} \right)^2 \right).
    \label{eq:I bound rt2}
\end{align}
Since the exact behavior of $I(f(\mathbf{x}_{:T}),f(\Tilde{\mathbf{x}}_{:T} ); \mathbf{y}_{:T}, \Tilde{\mathbf{y}}_{:T})$ is difficult to analyze, we use the universal upper bound as \citet{srinivas2012information}

\begin{align}
    I(f(\mathbf{x}_{:T}),f(\Tilde{\mathbf{x}}_{:T} ); \mathbf{y}_{:T}, \Tilde{\mathbf{y}}_{:T})
    =& I (\Tilde{\mathbf{y}}_{:T}; f(\mathbf{x}_{:T}),f(\Tilde{\mathbf{x}}_{:T}) \nonumber \\
    &+ \sum_{i=1}^n I(\mathbf{y}^i_{:T}; f(\mathbf{x}_{:T}),f(\Tilde{\mathbf{x}}_{:T} ) | \mathbf{y}^1_{:T}, \cdots, \mathbf{y}^{i-1}_{:T}, \Tilde{\mathbf{y}}_{:T}) \nonumber\\
    \leq& I(\Tilde{\mathbf{y}}_{:T} ; f(\Tilde{\mathbf{x}}_{:T})) +\sum_{i=1}^n I(\mathbf{y}^i_{:T}; f(\mathbf{x}^i_{:T})) \nonumber \\
    \leq& (n+a) \max_{\mathcal{A} \subset \mathcal{D}: |\mathcal{A}|=T} I(\mathbf{y}_{\mathcal{A}};f_{\mathcal{A}}).
    \label{eq:universal bound}
\end{align}
Setting $\gamma_T \triangleq \max_{\mathcal{A} \subset \mathcal{D}: |\mathcal{A}|=T} I(\mathbf{y}_{\mathcal{A}};f_{\mathcal{A}})$, 
\begin{align}
    \gamma_T &=I(\mathbf{y}_{\mathcal{A}};f_{\mathcal{A}}) = \frac{1}{2} \log\left|  \mathbf{I}+\frac{1}{\sigma^2}\mathbf{K}_{\mathcal{A}} \right|,
\end{align}
and combining \eqref{upperbound sum rt^2}, \eqref{eq:I bound rt2} and \eqref{eq:universal bound}, we can obtain
\begin{align}
    \sum_{t=1}^T (r_t)^2 &\leq c_1 (n+1) \beta_T \gamma_T.
\end{align}
Per Cauchy-Schwarz inequality, 
\begin{align}
    R_t &= \sum_{t=1}^T r_t \leq \sqrt{c_1 (n+1)T \beta_T \gamma_T}.
\end{align}

\section{Information Gain $\gamma_T$ for RFGP} \label{inf_gain_RFGP}
Since $\mathbf{K}_{\mathcal{A}}$ is approximated as $\sigma^2_{\theta} \mathbf{\Phi}_T {\mathbf{\Phi}_T}^\top$ 
for RFGP, the $\gamma_T$ is 
\begin{align}
    \gamma_T  = \log \left|  \mathbf{I}+\frac{\sigma_{\theta}^2}{\sigma^2}  \mathbf{\Phi}_T {\mathbf{\Phi}_T}^\top \right|. \label{eq:RFGP_info_gain}
\end{align}
Based on the determinant of a square matrix, 
\begin{align}
    \left|  \mathbf{I}+\frac{1}{\sigma^2}  \mathbf{\Phi}_T {\mathbf{\Phi}_T}^\top \right| &\leq \left(\frac{\text{trace}(\mathbf{I}+ c_3  \mathbf{\Phi}_T {\mathbf{\Phi}_t}^\top)}{2D}\right)^{2D} \nonumber\\
    & = \left(\frac{\text{trace}(\mathbf{I})+ c_3 \text{trace}(\mathbf{\Phi}_T {\mathbf{\Phi}_T}^\top)}{2D}\right)^{2D}, \label{eq:det_2D}
\end{align}
where $c_3\triangleq \frac{\sigma_{\theta}^2}{\sigma^2}$, $\text{trace}(\mathbf{I}) = 2D$ and 
\begin{align*}
     \text{trace}(\mathbf{\Phi}_t {\mathbf{\Phi}_t}^\top) = \sum_{i=1}^{2D} k_{ii}&= \sum_{i=1}^{2D} \frac{1}{D}(\phi_{\mathbf{v}}^\top(\mathbf{x}_1))_i(\phi_{\mathbf{v}}(\mathbf{x}_1))_i + \cdots + \frac{1}{D} (\phi_{\mathbf{v}}^\top(\mathbf{x}_1))_T(\phi_{\mathbf{v}}(\mathbf{x}_1))_T \\
     &\leq \frac{1}{D} \sum_{i=1}^{2D} N T = \mathcal{O}(2T).
\end{align*}
By combining them, we can get
\begin{align}
    \gamma_T &\leq \log\left(\left(\frac{D+ c_3 T}{D}\right)^{2D}\right) = \mathcal{O}\left(2D\log\left(\frac{T}{D}\right) \right). \label{eq:gamma_T_RF}
\end{align}

Thus, the information gain,$\gamma_T$, scales sub-linearly for RFGP and O-RFGP.

\section{Convergence of O-RFGP} \label{converge_ORFGP}
Without noisy measurements, $\left| f({x}) - \mu_{t}({x}) \right| = \mathcal{O}(\frac{1}{\sqrt{T}})$ for GP \citep{srinivas2009gaussian} and RFGP \citep{rahimi2007random}. As for O-RFGP, we will additionally show that $ \left| f({x}) - \mu_{t}({x}) \right| = \mathcal{O}(\frac{\log T}{T})$ in this section. For simplicity and clarity of proof, we set $n=1$ then $\mathbf{x}_t$ can be rewritten as $x_t$.

\begin{lemma} \label{lemma:ORFGP}
For O-RFGP,
    \begin{align}
        \left| f({x}) - \mu_{t}({x}) \right| = \mathcal{O}(\frac{\log T}{T}).
    \end{align}
\end{lemma}

To prove Lemma \ref{lemma:ORFGP}, we define the negative log-likelihood loss and regret for prediction 
\begin{align}
    \mathcal{L}(\mu({x}; \mathbf{\theta}); f({x})) &\triangleq -\log p(f({x}) | \mu({x}; \mathbf{\theta})) \\
    R_{\text{pred}}(T) &\triangleq \sum_{t=1}^T\mathcal{L}(\mu({x}_t; \mathbf{\theta}); f({x}_t)) - \sum_{t=1}^T \mathcal{L}(\mu({x}; \mathbf{\theta}^*); f({x}_t))
\end{align}
where $\mathbf{\theta}^*$ denotes the optimal solution and $\mu({x}; \mathbf{\theta})$ is computed as \eqref{eq:O-RFGP}. We omitted the superscript `O' here for clarity and readability.
Then, we need to introduce additional assumptions of \citet{lu2020ensemble}:
\begin{assumption}
    The loss, $\mathcal{L}$, is continuously twice differentiable with $|\frac{d^2}{d z_t^2} \mathcal{L}(z_t;f({x}_t)| \leq c_4$. \label{assump_D1}
\end{assumption}
\begin{assumption}
    The loss, $\mathcal{L}$, is convex and has bounded derivative $|\frac{d}{d z_t} \mathcal{L}(z_t;f({x}_t))| \leq L$ \label{assump_D2}
\end{assumption}
\begin{assumption}
    The kernel, $\Bar{k}$, is shift-invariant, standardized and bounded as $\Bar{k}({x}, {x}') \leq 1, \forall {x}, {x}'$ \label{assump_D3}
\end{assumption}

To establish the regret bound for O-RFGP, we also need the following lemma.
\begin{lemma} \label{lemma:D2}
    Under Assumptions \ref{assump_D1}, for a fixed $\epsilon_2>0$, the following bound holds with probability at least $1- 2^8 (\frac{\sigma_*}{\epsilon_2})^2 \exp(\frac{-D \epsilon_2^2}{4d+8})$
    \begin{align}
        \sum_{t=1}^T \mathcal{L} &\leq \sum_{t=1}^T \mathcal{L}(\phi_{\mathbf{v}}({x}_t)\mathbf{\theta}^*;f({x}_t))  + \frac{\lVert\mathbf{\theta}^*\rVert^2}{2 \sigma_{\theta}} + D \log(1+ \frac{T \sigma_{\theta}^2}{2D}).
    \end{align}
\end{lemma}
\begin{proof}
Defining the cumulative loss over T slots with a fixed $\mathbf{\theta}$ as
\begin{align}
    \mathcal{L}_{\mathbf{\theta}} \triangleq -\log (f(\mathbf{x}_{:T})|\mathbf{\theta};\mathbf{x}_{:T}) = \sum_{t=1}^T \mathcal{L}(\phi_{\mathbf{v}}^\top ({x}_t))\mathbf{\theta};f({x}_t)),
\end{align}
the expected cumulative loss over $q(\mathbf{\theta})$, a pdf of the fixed strategy $\mathbf{\theta}$, can be defined as
\begin{align}
    \Bar{ \mathcal{L}}_{q_{\mathbf{\theta}}} \triangleq \mathbb{E}_q [\mathcal{L}_{\mathbf{\theta}}] = \int_\mathbf{\theta} q(\mathbf{\theta}) \mathcal{L}_{\mathbf{\theta}} d\mathbf{\theta} . 
\end{align}
Based on Bayes rule, we have
\begin{align}
    \sum_{t=1}^T   \mathcal{L}(\mu({x}; \mathbf{\theta}); f({x})) &= \sum_{t=1}^T -\log (f({x}_t)| f({x}_{t-1}); \mathbf{x}_{:t}) \nonumber \\
    &= -\log p(f(\mathbf{x}_{:t});\mathbf{x}_{:t}).
\end{align}
Employing the definition of the Kullback-Leibler (KL) divergence, we have
\begin{align}
\label{eq:AppD_KL}
    \sum_{t=1}^T   \mathcal{L}(\mu({x}; \mathbf{\theta}); f({x})) - \Bar{ \mathcal{L}}_{q_{\mathbf{\theta}}}  = \int_\mathbf{\theta} q(\mathbf{\theta}) \log \frac{p(f(\mathbf{x}_{:t})| \mathbf{\theta};\mathbf{x}_{:t})}{p(f(\mathbf{x}_{:t});\mathbf{x}_{:t})} d\mathbf{\theta} .
\end{align}
Since $p(f(\mathbf{x}_{:t}), \mathbf{\theta};\mathbf{x}_{:t}) = p(f(\mathbf{x}_{:t})| \mathbf{\theta};\mathbf{x}_{:t}) p (\mathbf{\theta}) = p(f(\mathbf{x}_{:t}); \mathbf{x}_{:t}) p(\mathbf{\theta}|f(\mathbf{x}_{:t});\mathbf{x}_{:t} )$, the RHS of \eqref{eq:AppD_KL} can be rewritten as 
\begin{align*}
    \int_\mathbf{\theta} q(\mathbf{\theta}) \log \frac{p(f(\mathbf{x}_{:t})| \mathbf{\theta};\mathbf{x}_{:t})}{p(f(\mathbf{x}_{:t});\mathbf{x}_{:t})} d\mathbf{\theta}
    &= \int_\mathbf{\theta} q(\mathbf{\theta}) \log \frac{p(\mathbf{\theta}|f(\mathbf{x}_{:t});\mathbf{x}_{:t} )}{p(\mathbf{\theta})} d\mathbf{\theta}\\
    &= \int_\mathbf{\theta} q(\mathbf{\theta}) \log\frac{q(\mathbf{\theta})}{p(\mathbf{\theta})} d\mathbf{\theta} 
    - \int_\mathbf{\theta} q(\mathbf{\theta}) \log\frac{q(\mathbf{\theta})}{p(\mathbf{\theta}|f(\mathbf{x}_{:t});\mathbf{x}_{:t} ) } d\mathbf{\theta}\\
    &\leq KL(q(\mathbf{\theta}) || p(\mathbf{\theta})).
\end{align*}
Hence, it holds for any $q(\mathbf{\theta})$ (\textit{Lemma 2} of \citet{lu2020ensemble}) that
\begin{align}
    \sum_{t=1}^T   \mathcal{L}(\mu({x}; \mathbf{\theta}); f({x}))\leq \Bar{ \mathcal{L}}_{q_{\mathbf{\theta}}}  + KL(q(\mathbf{\theta}) || p(\mathbf{\theta})). \label{lemma2_lud}
\end{align}
Let $q(\mathbf{\theta}) = \mathcal{N}(\mathbf{\theta};\mathbf{\theta}^*, \xi^2 \mathbf{I}_{2D})$ and $p(\mathbf{\theta})= \mathcal{N}(\mathbf{\theta};\mathbf{0}, \sigma_{\theta}^2 \mathbf{I}_{2D})$, combined with \eqref{lemma2_lud}, we can have
\begin{align}
    \sum_{t=1}^T   \mathcal{L}(\mu({x}; \mathbf{\theta}); f({x})) - \Bar{ \mathcal{L}}_{q_{\mathbf{\theta}}} &\leq KL(q(\mathbf{\theta}) || p(\mathbf{\theta})) \nonumber\\
    &~ = 2D \log \sigma_{\theta} + \frac{1}{2\sigma_{\theta}} (\lVert\mathbf{\theta}^*\rVert^2 + 2D \xi^2) -2D -2D\log\xi.\label{eq:L-Lq}
\end{align}
Let $z_t = \phi_{\mathbf{v}}^\top({x}_{t}) \mathbf{\theta}$ and $z_t^* = \phi_{\mathbf{v}}^\top({x}_{t}) \mathbf{\theta}^*$. Taking Taylor's expansion of $\mathcal{L}(z_t; f({x}))$ around $z_t^*$ results in
\begin{align}
    \mathcal{L}(z_t; f({x})) &= \mathcal{L}(z^*_t; f(f{x})) + \frac{d \mathcal{L}(z^*_t; f({x})) }{d z_t} (z_t - z^*_t) + \frac{d^2}{d z_t^2 } \mathcal{L}(h(z_t); f({x})) \frac{ (z_t - z^*_t)^2}{2}, \label{eq:Taylor,z_t}
\end{align}
where $h(z_t)$ is some function lying between $z_t$ and $z^*_t$. Taking the expectation of \eqref{eq:Taylor,z_t} with respect to $q(\mathbf{\theta})$ leads to
\begin{align}
    \mathbb{E}_q &= \mathcal{L}(z^*_t; f({x})) + \frac{d \mathcal{L}(z^*_t; f({x})) }{d z_t} \times 0 \nonumber+ \mathbb{E}_q \left[ \frac{d^2}{d z_t^2 } \mathcal{L}(h(z_t); f({x})) \frac{ (z_t - z^*_t)^2}{2}\right] \nonumber\\
    &\stackrel{(a)}{\leq} \mathcal{L}(z^*_t; f({x})) + c_4 \mathbb{E}\left[ \frac{ (z_t - z^*_t)^2}{2}\right] \nonumber\\
    & \stackrel{(b)}{\leq}  \mathcal{L}(\phi_{\mathbf{v}}^\top({x}_{t}) \mathbf{\theta}; f({x})) + \frac{c_4 \xi^2}{2}, \label{eq:Eq}
\end{align}
where $(a)$ makes use of \textit{Assumption \ref{assump_D1}}, and $(b)$ relies on the bound $\lVert \phi_{\mathbf{v}}({x}_{t}) \rVert^2 \leq 1$.
Summing \eqref{eq:Eq} from $t=1$ to $T$, we have
\begin{align}
    \Bar{ \mathcal{L}}_{q_{\mathbf{\theta}}} \leq \mathcal{L}_{\mathbf{\theta}^*} + \frac{T c_4 \xi^2}{2}.
\end{align}
Combining with \eqref{eq:L-Lq}, the following inequality holds
\begin{align}
    &\sum_{t=1}^T   \mathcal{L}(\mu({x}; \mathbf{\theta}); f({x})) \leq \mathcal{L}_{\mathbf{\theta}^*} + \frac{T c_4 \xi^2}{2} + 2D \log \sigma_{\theta} + \frac{1}{2\sigma_{\theta}} (\lVert\mathbf{\theta}^*\rVert^2 + 2D \xi^2) -2D -2D\log\xi. \label{eq:with xi^2}
\end{align}
Since the RHS is a convex function of $\xi$ with minimal value at $\xi^2 =\frac{2D \sigma_{\theta}^2}{2D+T c_4 \sigma_{\theta}^2}$, replacing $\xi$ with the minimal value in \eqref{eq:with xi^2} leads to
\begin{align*}
    \sum_{t=1}^T   \mathcal{L}(\mu({x}; \mathbf{\theta}); f({x})) &\leq \sum_{t=1}^T \mathcal{L}(\phi_{\mathbf{v}}({x}_t)\mathbf{\theta}^*;f({x}_t))+ \frac{1}{2\sigma_{\theta}} \lVert\mathbf{\theta}^*\rVert^2 + D \log\left(1 +\frac{T c_4 \sigma_{\theta}^2}{2D}\right).
\end{align*}
Proof of Lemma \ref{lemma:D2} is completed.
\end{proof}

Lemma \ref{lemma:D2} bounds the cumulative Bayesian loss. Next, we will further bound the loss of O-RFGP estimator relative to the best function estimator in the original RKHS.

\begin{theorem} \label{theorem2}
    Under \textit{Assumptions \ref{assump_D1}, \ref{assump_D2} and \ref{assump_D3}}, and with optimal function estimator ($\hat{f}^*$) belonging to the RHKS $\mathcal{H}^*$ induced by $k^*$,
    for a fixed $\epsilon_2>0$,  the following bound holds with probability at least $1- 2^8 (\frac{\sigma_*}{\epsilon_2})^2 \exp(\frac{-D \epsilon_2^2}{4d+8})$
    \begin{align}
       \sum_{t=1}^T   \mathcal{L}(\mu({x}; \mathbf{\theta}); f({x})) - \sum_{t=1}^T \mathcal{L}(\hat{f}^*({x}_t);f({x}_t))  \leq 
       \frac{(1+\epsilon_2)^2}{2\sigma_{\theta^2}} + D \log\left(1 +\frac{T c_4 \sigma_{\theta}^2}{2D}\right) + \epsilon_2 LTC,
    \end{align}
    where $C$ is a constant, $d=2$ in our case.
    Let $\epsilon_2=\mathcal{O}(\frac{\log T}{T})$, $R_{\text{pred}}(T) = \mathcal{O}(\log T)$.
\end{theorem}
\begin{proof}
By \citet{rahimi2007random}, for a given shift-invariant standardized kernel $\Bar{k}$, the maximum point-wise error of the RF kernel approximate is uniformly bounded with probability at least $1- 2^8 (\frac{\sigma_*}{\epsilon_2})^2 \exp(\frac{-D \epsilon_2^2}{4d+8})$,
\begin{align}
    \sup_{{x}, {x}'}  \left| \phi_{\mathbf{v}}^\top({x}) \phi_{\mathbf{v}}({x}') - \Bar{k}({x}, {x}') \right| \leq \epsilon_2, \label{eq:I,rahimi}
\end{align}
where $\epsilon_2$ is a given constant, $D$ is the number of the feature vectors, $d$ is dimension of ${x}$, and $\sigma_*^2 \triangleq \mathbb{E}\left[ \lVert \mathbf{v}\rVert^2 \right]$ is the second-order moment of the RF vector $\mathbf{v}$.

The optimal function estimator in $\mathcal{H}^*$ incurred by $k^*$ is expressed as $\hat{f}({x}) \triangleq \sum_{t=1}^T \hat{\alpha}_t k^*({x},{x}_t) = \sigma_{\theta}^2 \sum_{t=1}^T \hat{\alpha}_t \Bar{k}({x},{x}_t)$; and its RF-based approximate is $\check{f}^*({x}) \triangleq \phi_{\mathbf{v}}^\top({x})\theta^*$ with $\mathbf{\theta}^* \triangleq \sigma_{\theta}^2 \sum_{t=1}^T \hat{\alpha}_t \phi_{\mathbf{v}}({x})$. Then, we can have
\begin{align}
    &\left| \sum_{t=1}^T\mathcal{L}(\check{f}^*({x}_t);f({x}_t)) - \sum_{t=1}^T\mathcal{L}(\hat{f}^*({x}_t); f({x}_t)) \right| \nonumber\\
    \stackrel{(a)}{\leq}& \sum_{t=1}^T \left| \mathcal{L}(\check{f}^*({x}_t);f({x}_t)) - \mathcal{L}(\hat{f}^*({x}_t); f({x}_t)) \right| \nonumber\\
    \stackrel{(b)}{\leq}& \sum_{t=1}^T L \sigma_{\theta}^2 \left| \sum_{t'=1}^T \hat{\alpha}_{t'} \phi_{\mathbf{v}}^\top({x}_t) \phi_{\mathbf{v}}({x}_{t'}) - \sum_{t'=1}^T \hat{\alpha}_{t'} \Bar{k}({x}_t,{x}_{t'})  \right| \nonumber\\
    \stackrel{(c)}{\leq}& \sum_{t=1}^T L \sigma_{\theta}^2 \sum_{t'=1}^T |\hat{\alpha}_{t'}| \left| \phi_{\mathbf{v}}^\top({x}_t) \phi_{\mathbf{v}}({x}_{t'}) - \Bar{k}({x}_t,{x}_{t'}) \right|, \label{eq:II}
\end{align}
where $(a)$ follows the triangle inequality, $(b)$ makes use of \textit{Assumption \ref{assump_D2}}, and $(c)$ results from the Cauchy-Schwarz inequality. Combining \eqref{eq:II} with \eqref{eq:I,rahimi}, we find that
\begin{align}
   \left| \sum_{t=1}^T\mathcal{L}(\check{f}^*({x}_t);f({x}_t)) - \sum_{t=1}^T\mathcal{L}(\hat{f}^*({x}_t); f({x}_t)) \right| 
   \leq \sum_{t=1}^T L \sigma_{\theta}^2 \epsilon_2 \sum_{t'=1}^T|\hat{\alpha}_{t'}| \leq \epsilon_2 LTC, 
\end{align}
where $C\triangleq \sum_{t=1}^T \sigma_{\theta}^2 |\hat{\alpha}_t|$. Thus, we can have
\begin{align}
    \left| \sum_{t=1}^T\mathcal{L}(\phi_{\mathbf{v}}^\top({x}_t) \mathbf{\theta}^*;f({x}_t)) - \sum_{t=1}^T\mathcal{L}(\hat{f}^*({x}_t); f({x}_t)) \right| \leq \epsilon_2 LTC. \label{eq:III}
\end{align}
The uniform convergence bound \eqref{eq:I,rahimi} and \textit{Assumption \ref{assump_D3}} imply that
\begin{align}
    \sup_{{x}, {x}'}   \phi_{\mathbf{v}}^\top({x}) \phi_{\mathbf{v}}({x}')   \leq \epsilon_2 +1,
\end{align}
which leads to 
\begin{align}
    \lVert \mathbf{\theta}^* \rVert^2 &\triangleq \left\lVert \sigma_{\theta}^2 \sum_{t=1}^T \hat{\alpha}_t \phi_{\mathbf{v}}({x}_t) \right\rVert^2 \nonumber\\
    &= \sigma_{\theta}^4 \sum_{t=1}^T \sum_{t'=1}^T \hat{\alpha}_t \hat{\alpha}_{t'} \phi^\top_{\mathbf{v}}({x}_t)\phi_{\mathbf{v}}({x}_{t'}) \nonumber\\
    &\leq (1+\epsilon_2) C^2. \label{eq:IV}
\end{align}
Combing \eqref{eq:III}, \eqref{eq:IV} and Lemma \ref{lemma:D2}, we can have
\begin{align}
    \sum_{t=1}^T   \mathcal{L}(\mu({x}; \mathbf{\theta}); f({x})) - \sum_{t=1}^T \mathcal{L}(\hat{f}^*({x}_t);f({x}_t))  
    \leq \frac{(1+\epsilon_2)C^2}{2\sigma_{\theta}^2} + D\log\left(1+\frac{T c_4 \sigma_{\theta}^2 }{2D }\right) + \epsilon_2 LTC.
\end{align}
Proof of Theorem \ref{theorem2} is completed.
\end{proof}

Theorem \ref{theorem2} provides the upper bound of prediction for O-RFGP: $R_{\text{pred}}(T) = \mathcal{O}(\log T)$ while $\mathcal{L}$ is Bayesian loss. Assuming the loss of optimal estimator $\hat{f}^*$ is relatively negligible,  $\mathcal{L}(\hat{f}^*({x}_t);f({x}_t)) \approx0$, we can have $\sum_{t=1}^T   \mathcal{L}(\mu({x}; \mathbf{\theta}); f({x})) = \mathcal{O}(\log T)$. Then, $\mathcal{L}(\mu_t({x}); f({x}))$ can be rewritten as
\begin{align*}
    \mathcal{L}(\mu_t({x}_t); f({x}_t)) &= -\log p(f({x}_t; \mathbf{\theta}) | \mu_t({x}_t)) \\
    &= \frac{(f({x}_t) - \mu_t({x}_t))^2}{2 \sigma_t^2} + \frac{1}{2} \log (2 \pi \sigma_t^2).
\end{align*}
Since $\sigma_t^2$ is bounded due to the bounded kernel, the cumulative squared error also scales as $\mathcal{O}(\log T)$:  
by Cauchy-Schwarz inequality,  we can have
\begin{align*}
    \sum_{t=1}^T |f({x}_t) - \mu_t({x}_t)| &\leq \sqrt{T \sum_{t=1}^T (f({x}_t) - \mu_t({x}_t))^2}\\
    & = \mathcal{O}(\sqrt{T \log T}).
\end{align*}
Since the cumulative absolute error is $\mathcal{O}(\sqrt{T \log T})$, the per-step error is 
\begin{align*}
    \left| f({x}) - \mu_{t}({x}) \right| = \mathcal{O}(\sqrt{\frac{\log T}{T}}).
\end{align*}
Although it converges slower than $\mathcal{O}({\frac{1}{\sqrt{T}}})$, it is still fast enough for the requirement in Lemma \ref{lemma1}.

\section{Experimental Settings and Implementation Details}
\label{sec:experimentalsettings}

The domain of interest $\mathcal{D}$ is a rectangular domain with the origin located at the center of it, $x$-axis ranging from $-1.6$ to $1.6$ meters, and $y$-axis ranging from $-1$ to $1$ meter. The common parameters for the simulations are set as follows. The initial positions of the team of robots are ${x}^1_0 = [-0.8,-0.2]^\top$, ${x}^2_0 = [-0.8,0.2]^\top$, ${x}^3_0 = [-0.4,-0.2]^\top$, ${x}^4_0 = [-0.4,0.2]^\top$, 
${x}^5_0 = [0,-0.2]^\top$, ${x}^6_0 = [0,0.2]^\top$, ${x}^7_0 = [0.4,-0.2]^\top$, ${x}^8_0 = [0.4,0.2]^\top$, ${x}^9_0 = [0.8,-0.2]^\top$, and ${x}^{10}_0 = [0.8,0.2]^\top$. The initialization of the learned density is a uniform distribution across $\mathcal{D}$. Each robot takes a noisy measurement of the density value at its current position at each time step. The parameters of RFGP for learning the mean and variance of the density function are $\tau_{\text{RF}} = 0.1$, $\sigma_{\theta}^2 = 5$. The non-decreasing sequence $\sqrt{\beta_t}$ is set as  $10^{-3}$ for RFGP and $10^{5}$ for O-RFGP. The noise $\epsilon$ is sampled from a zero-mean Gaussian distribution with $\sigma = 0.01$. The number of random features are $D = 20$ and $D = 40$ for RFGP and O-RFGP, respectively. 

The time-invariant density function is defined as
$
f({x}) = \sum_{i=1}^6 A_i \exp (-({x} - \textbf{m}_i)^\top S_i ({x} - \textbf{m}_i))$,
where the amplitudes are $A_1 = A_2 = A_3 = A_4 = 150$, $A_5 = A_6 = 50$, the means are $\textbf{m}_1 = [1.3,-0.75]^\top$, $\textbf{m}_2 = [-1.2,0.7]^\top$, $\textbf{m}_3 = \textbf{m}_6 = [1,0.8]^\top$, $\textbf{m}_4 = \textbf{m}_5 = [-1,-0.8]^\top$, and the covariance matrices are $S_1 = S_2 = \text{diag}(0.3^2,0.8^2)$, $S_3 = \text{diag}(0.8^2,0.4^2)$, $S_4 = \text{diag}(0.4^2,0.8^2)$, $S_5 = S_6 = \text{diag}(0.85^2,0.2^2)$. 

The time-varying density function is defined as $
f_t({x}) = \! \sum_{i=1}^6 A_i(t) \cdot \exp (-({x} - \textbf{m}_i)^{\!\top} \! S_i ({x} - \textbf{m}_i)$,
where the amplitudes are $A_1 = A_2 = A_3 = A_4 = 150$, $A_5(t) = 50 (1+\sin(0.12t))$, $A_6(t) = 50 (1-\sin(0.12t))$, the means are $\textbf{m}_1 = [1.3,-0.75]^\top$, $\textbf{m}_2 = [-1.2,0.7]^\top$, $\textbf{m}_3 = \textbf{m}_6 = [1,0.8]^\top$, $\textbf{m}_4 = \textbf{m}_5 = [-1,-0.8]^\top$, and the covariance matrices are $S_1 = S_2 = \text{diag}(0.3^2,0.8^2)$, $S_3 = \text{diag}(0.8^2,0.4^2)$, $S_4 = \text{diag}(0.4^2,0.8^2)$, $S_5 = S_6 = \text{diag}(0.85^2,0.2^2)$. 

\section{Additional Experiments}
\label{sec:additionalexp}
Apart from Section~\ref{sec:simulations}, we perform several additional simulations for: 1. sensitivity study of the number of random features $D$ and non-decreasing sequence $\beta_t$, and 2. comparison of the computational efficiency among GP, RFGP, and O-RFGP.


\subsection{Sensitivity Study} \label{sec:sensitive}

\subsubsection{Dimension of random features $D$}

\begin{figure}[htbp]
\centering
\subfigure[]{
\begin{minipage}[b]{0.35\textwidth}
\includegraphics[width=1\textwidth]{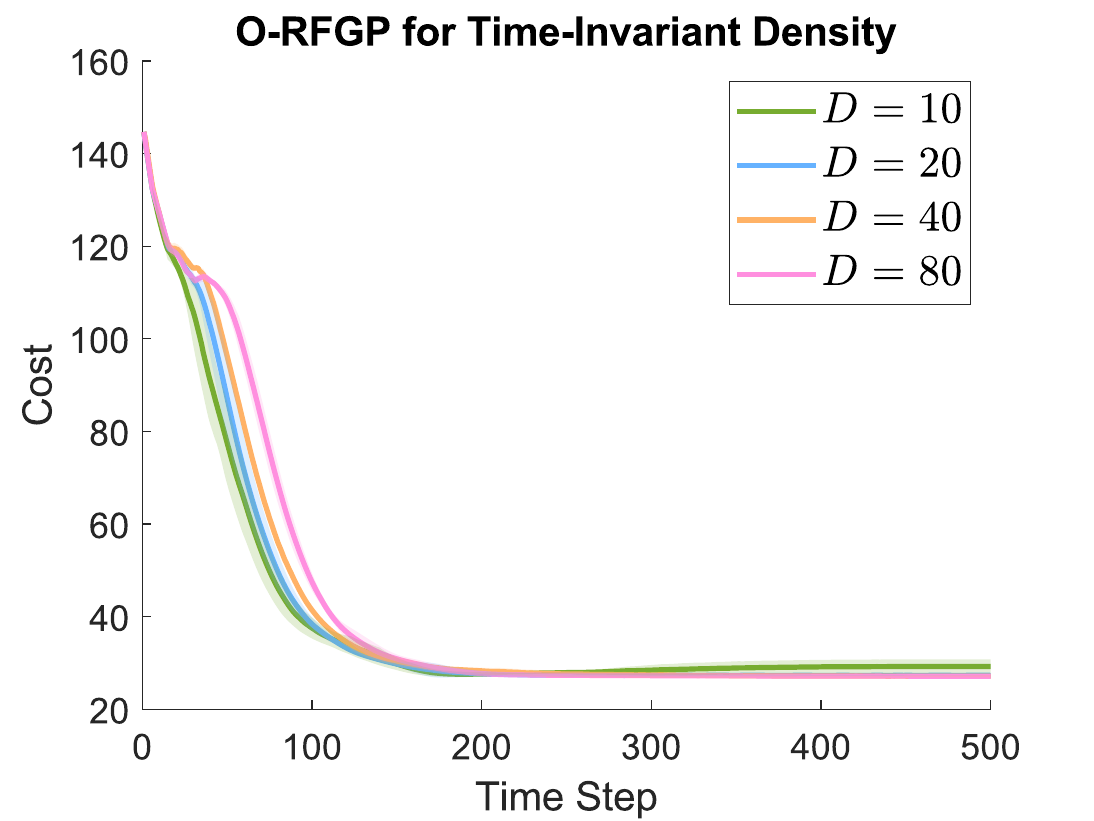}
\end{minipage}
\label{Sensitivity:D,TI,Cost}
}
\subfigure[]{
\begin{minipage}[b]{0.35\textwidth}
\includegraphics[width=1\textwidth]{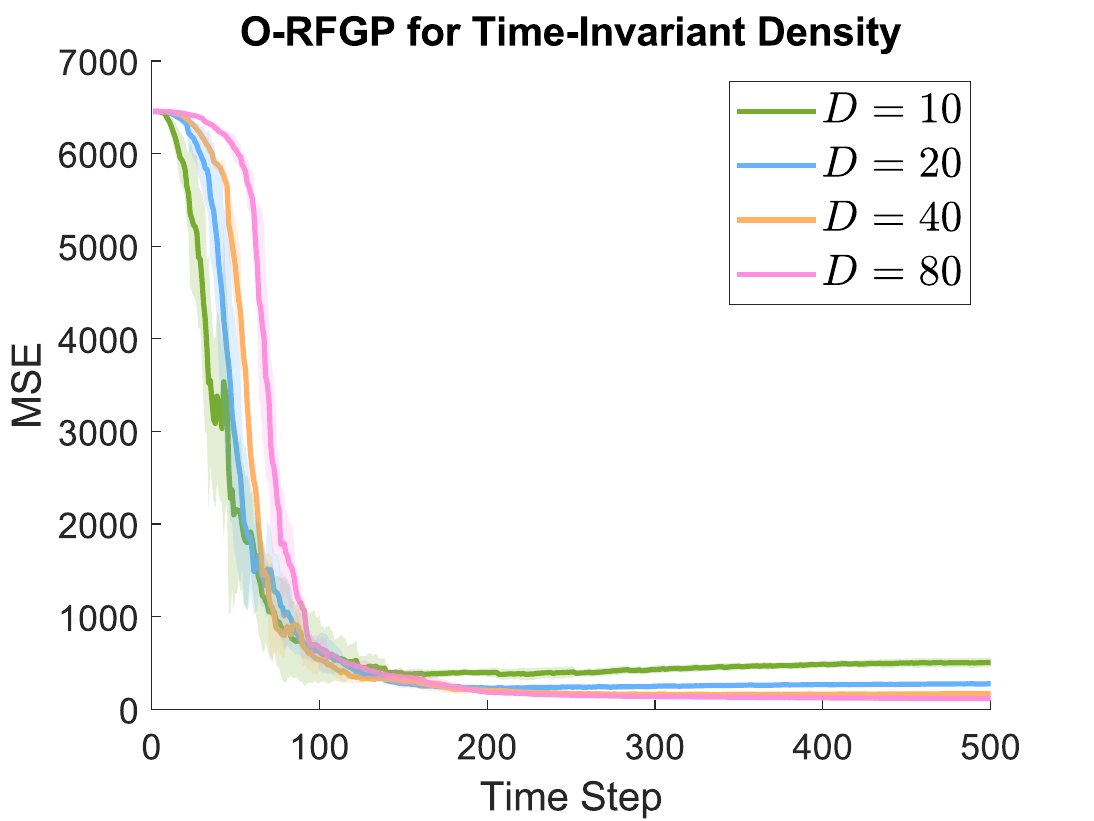}
\end{minipage}
\label{Sensitivity:D,TI,MSE}
}
\caption{Sensitivity study of $D$ when using O-RFGP to estimate the time-invariant density.
}
\label{Sensitivity:D,TI}
\end{figure}

\begin{table}[htbp]
\centering
\caption{Cost when using O-RFGP with different values of $D$ to estimate the time-invariant density. Associated with Fig.~\ref{Sensitivity:D,TI,Cost}.}
\begin{tabular}{lcccc}
\toprule
\multicolumn{1}{c}{$D$} & \textbf{$T=50$} & \textbf{$T=100$} & \textbf{$T=200$} & \textbf{$T=500$} \\ 
\midrule
\textbf{$10$} & $77.61\pm 8.92$ & $37.57 \pm 2.25$ & $ 27.63 \pm 0.39$ & $29.25 \pm 1.52$ \\
\textbf{$20$} & $87.28 \pm 8.77$ & $38.43 \pm 1.46$ & $27.72 \pm 0.62$ & $27.34 \pm 0.02$ \\
\textbf{$40$} & $96.35 \pm 2.46$ & $41.47 \pm 0.53$ & $28.32 \pm 0.68$ & $27.15 \pm 0.22$ \\
\textbf{$80$} & $108.18 \pm 1.7754$ & $47.57 \pm 1.88$ & $27.74 \pm 0.25$ & $27.14 \pm 0.04$ \\
\bottomrule
\end{tabular}
\label{table:Sensitivity:D,TI,Cost}
\end{table}

\begin{table}[htbp]
\centering
\caption{MSE ($\times 10^3$) when using O-RFGP with different values of $D$ to estimate the time-invariant density.. Associated with Fig.~\ref{Sensitivity:D,TI,MSE}.}
\begin{tabular}{lcccc}
\toprule
\multicolumn{1}{c}{$D$} & \textbf{$T=50$} & \textbf{$T=100$} & \textbf{$T=200$} & \textbf{$T=500$} \\ 
\midrule
\textbf{$10$} 
& $2.15 \pm 0.91$
& $0.65 \pm 0.37$ 
& $0.40 \pm 0.04$ 
& $0.51 \pm 0.05$ \\
\textbf{$20$} 
& $2.90 \pm 1.24$ 
& $0.62 \pm 0.19$ 
& $0.23 \pm 0.01$ 
& $0.28 \pm 0.26$ \\
\textbf{$40$} 
& $4.78 \pm 0.53$ 
& $0.54 \pm 0.13$ 
& $0.20 \pm 0.04$ 
& $0.17 \pm 0.03$ \\
\textbf{$80$} 
& $6.05 \pm 0.15$ 
& $0.66 \pm 0.09$ 
& $0.19 \pm 0.02$ 
& $0.12 \pm 0.01$ \\
\bottomrule
\end{tabular}
\label{table:Sensitivity:D,TI,MSE}
\end{table}

\begin{figure}[htbp]
\centering
\subfigure[]{
\begin{minipage}[b]{0.35\textwidth}
\includegraphics[width=1\textwidth]{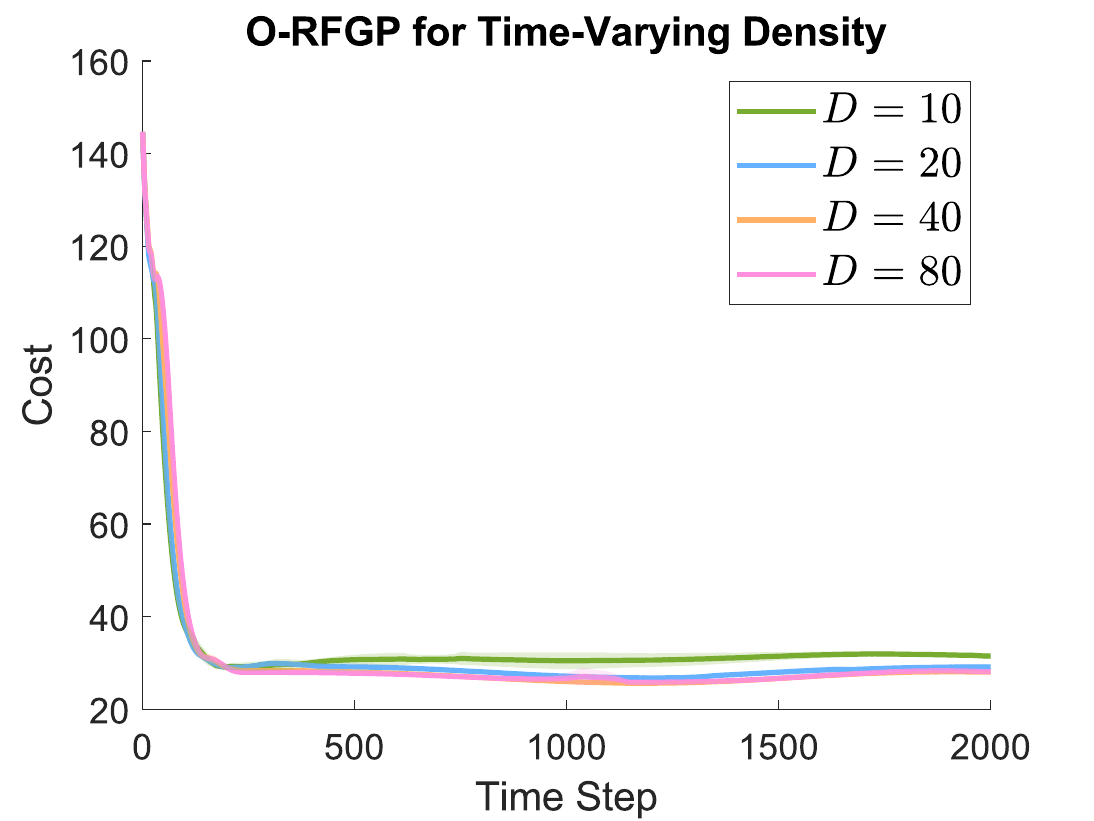}
\end{minipage}
\label{Sensitivity:D,TV,Cost}
}
\subfigure[]{
\begin{minipage}[b]{0.35\textwidth}
\includegraphics[width=1\textwidth]{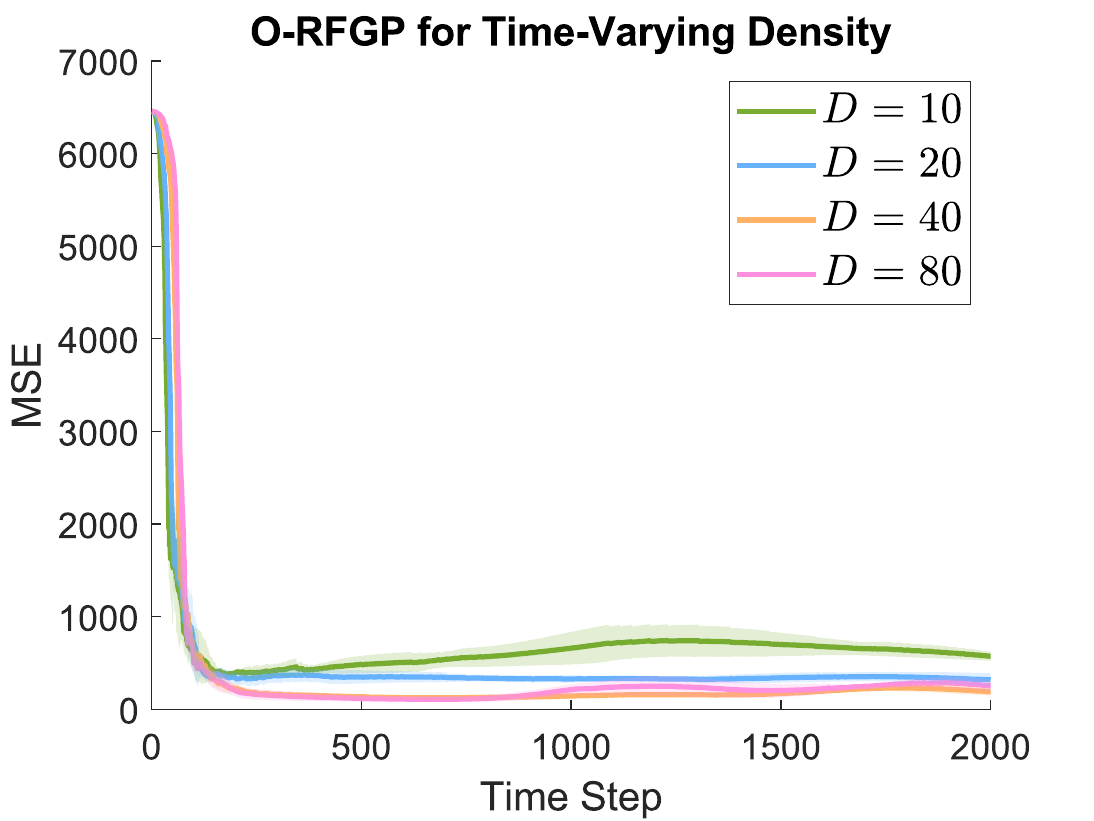}
\end{minipage}
\label{Sensitivity:D,TV,MSE}
}
\caption{Sensitivity study of $D$ when using O-RFGP to estimate the time-varying density.
}
\label{Sensitivity:D,TV}
\end{figure}

\begin{table}[htbp]
\centering
\caption{Cost when using O-RFGP with different values of $D$ to estimate the time-varying density. Associated with Fig.~\ref{Sensitivity:D,TV,Cost}.}
\begin{tabular}{lcccccc}
\toprule
\multicolumn{1}{c}{$D$} 
& \textbf{$T=50$} 
& \textbf{$T=100$} 
& \textbf{$T=500$}
& \textbf{$T=1000$}
& \textbf{$T=1500$}
& \textbf{$T=2000$}
\\ 
\midrule
\textbf{$10$} 
& $77.71 \pm 5.06$ 
& $37.95 \pm 1.50$  
& $30.72 \pm 1.06$
& $30.49 \pm 1.81$ 
& $31.48 \pm 0.88$ 
& $31.49 \pm 0.63$ 
\\
\textbf{$20$} 
& $82.42 \pm 1.35$ 
& $38.45 \pm 1.56$ 
& $29.18 \pm 0.80$ 
& $27.18 \pm 0.44$ 
& $28.02 \pm 0.61$ 
& $29.15 \pm 0.64$ 
\\
\textbf{$40$} 
& $97.14 \pm 5.16$ 
& $42.79 \pm 1.22$  
& $28.09 \pm 0.13$ 
& $26.00 \pm 0.21$ 
& $26.68 \pm 0.21$ 
& $28.04 \pm 0.27$ 
\\
\textbf{$80$} 
& $105.52 \pm 2.26$ 
& $45.11 \pm 1.09$  
& $27.81 \pm 0.23$ 
& $26.71 \pm 1.15$ 
& $26.70 \pm 0.24$ 
& $28.12 \pm 0.07$ 
\\
\bottomrule
\end{tabular}
\label{table:Sensitivity:D,TV,Cost}
\end{table}

\begin{table}[htbp]
\centering
\caption{MSE ($\times 10^3$) when using O-RFGP with different values of $D$ to estimate the time-varying density. Associated with Fig.~\ref{Sensitivity:D,TV,MSE}.}
\begin{tabular}{lcccccc}
\toprule
\multicolumn{1}{c}{$D$} 
& \textbf{$T=50$} 
& \textbf{$T=100$} 
& \textbf{$T=500$} 
& \textbf{$T=1000$}
& \textbf{$T=1500$} 
& \textbf{$T=2000$} 
\\ 
\midrule
\textbf{$10$} 
& $1.60 \pm 0.68$ 
& $0.64 \pm 0.20$ 
& $0.49 \pm 0.09$
& $0.66 \pm 0.18$ 
& $0.70 \pm 0.13$ 
& $0.57 \pm 0.05$ 
\\
\textbf{$20$} 
& $2.06 \pm 0.09$ 
& $0.71 \pm 0.31$ 
& $0.35 \pm 0.07$ 
& $0.33 \pm 0.04$ 
& $0.34 \pm 0.05$ 
& $0.32 \pm 0.07$ 
\\
\textbf{$40$} 
& $5.22 \pm 0.88$ 
& $0.57 \pm 0.07$ 
& $0.14 \pm 0.02$ 
& $0.14 \pm 0.03$ 
& $0.17 \pm 0.01$ 
& $0.19 \pm 0.04$ 
\\
\textbf{$80$} 
& $5.93 \pm 0.17$ 
& $0.56 \pm 0.12$ 
& $0.19 \pm 0.03$
& $0.21 \pm 0.06$ 
& $0.20 \pm 0.07$ 
& $0.25 \pm 0.06$ 
\\
\bottomrule
\end{tabular}
\label{table:Sensitivity:D,TV,MSE}
\end{table}

To analyze the influence of the number of random features $D$, the results in Fig.~\ref{Sensitivity:D,TI}, Table~\ref{table:Sensitivity:D,TI,Cost}, Table~\ref{table:Sensitivity:D,TI,MSE}, Fig.~\ref{Sensitivity:D,TV}, Table~\ref{table:Sensitivity:D,TV,Cost} and Table~\ref{table:Sensitivity:D,TV,MSE} show that large values of $D$ lead to accurate estimation of both time-invariant and time-varying densities, which in turn improves the overall coverage quality. 
Although the theoretical analysis holds only in the time-invariant setting, it suggests that similar principles also govern the performance of O-RFGP in the time-varying setting. In particular, as $D$ increases, the MSE decreases, consistent with the analysis presented in Appendix \ref{converge_ORFGP}.

\subsubsection{Non-decreasing Sequence $\beta_t$}

\begin{figure}[htbp]
\centering
\subfigure[]{
\begin{minipage}[b]{0.35\textwidth}
\includegraphics[width=1\textwidth]{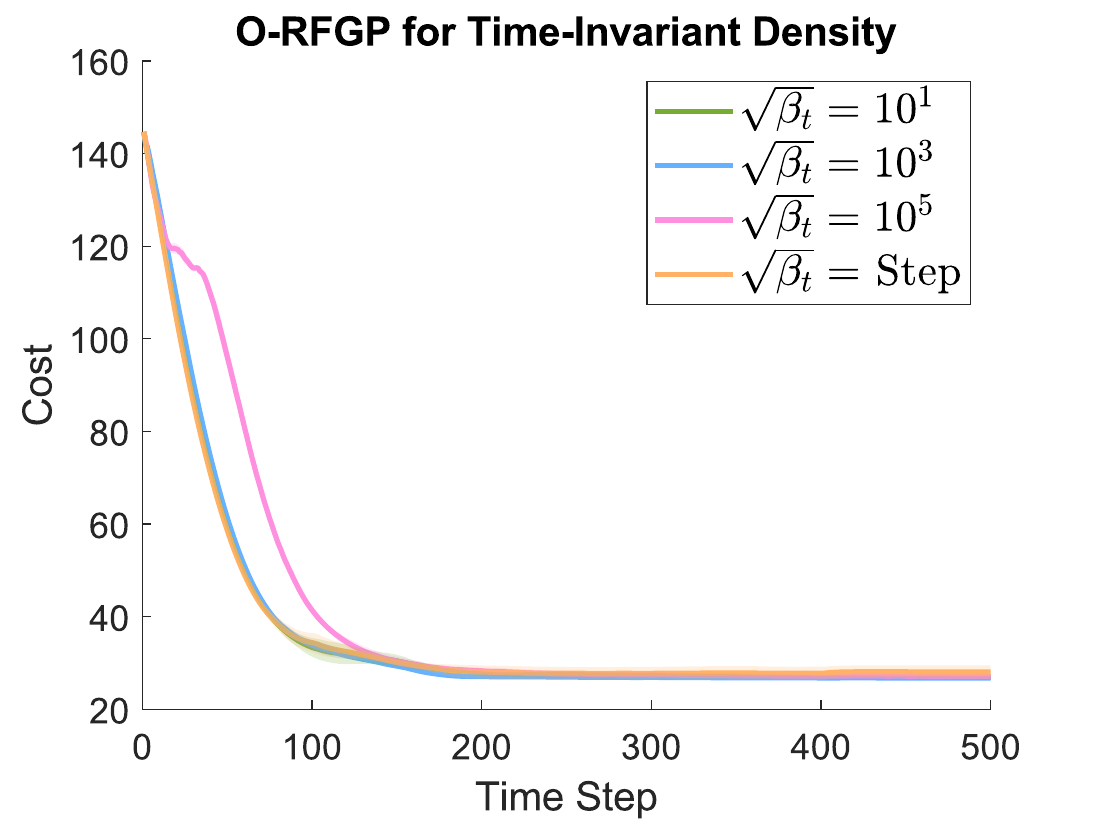}
\end{minipage}
\label{Sensitivity:beta,TI,Cost}
}
\subfigure[]{
\begin{minipage}[b]{0.35\textwidth}
\includegraphics[width=1\textwidth]{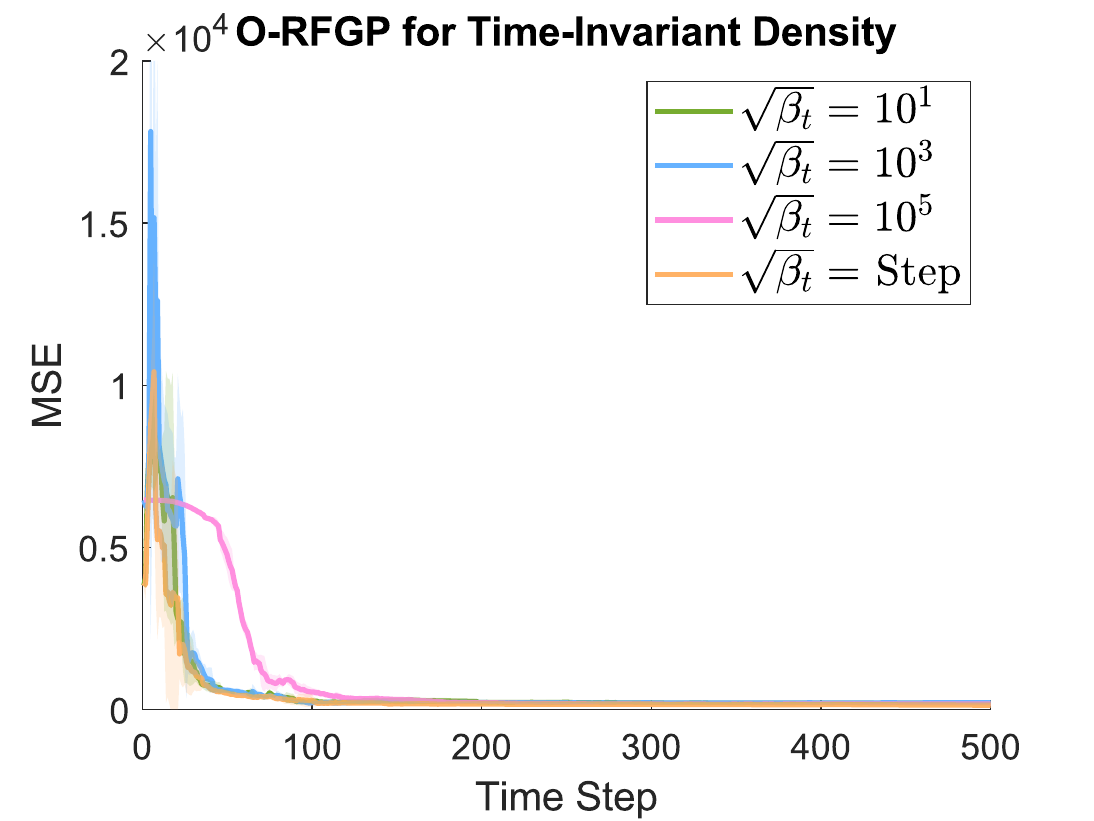}
\end{minipage}
\label{Sensitivity:beta,TI,MSE}
}
\caption{Sensitivity study of $\sqrt{\beta_t}$ when using O-RFGP to estimate the time-invariant density. ``$\sqrt{\beta_t}$ = Step'' denotes a piecewise constant schedule: $\sqrt{\beta_t} = 10^1$ for time steps $0$-$100$, $\sqrt{\beta_t} = 10^2$ for time steps $100$-$200$, $\sqrt{\beta_t} = 10^3$ for time steps $200$-$300$, $\sqrt{\beta_t} = 10^4$ for time steps $300$-$400$, and $\sqrt{\beta_t} = 10^5$ for time steps $400$-$500$.
}
\label{Sensitivity:beta,TI}
\vspace{-0.2cm}
\end{figure}

\begin{table}[htbp]
\vspace{-0.2cm}
\centering
\caption{Cost when using O-RFGP with different values of $D$ to estimate the time-invariant density. Associated with Fig.~\ref{Sensitivity:beta,TI,Cost}.}
\begin{tabular}{lcccc}
\toprule
\multicolumn{1}{c}{$\sqrt{\beta_t}$} & \textbf{$T=50$} & \textbf{$T=100$} & \textbf{$T=200$} & \textbf{$T=500$} \\ 
\midrule
\textbf{$10^1$} 
& $60.10 \pm 1.10$ 
& $33.41 \pm 2.00$ 
& $27.83 \pm 0.36$ 
& $26.92 \pm 0.21$ \\
\textbf{$ 10^3$}
& $61.53 \pm 0.53$ 
& $33.88 \pm  0.96$ 
& $27.09 \pm 0.53$ 
& $26.78 \pm  0.28$ \\
\textbf{$ 10^5$} 
& $96.35 \pm 2.46$ 
& $41.47 \pm 0.53$ 
& $28.32 \pm 0.68$ 
& $27.15 \pm 0.22$ \\
\textbf{$\mathrm{Step}$} 
& $58.81 \pm 0.11$ 
& $34.40 \pm 2.07$ 
& $28.07 \pm 1.42$ 
& $28.05 \pm 1.43$ \\
\bottomrule
\end{tabular}
\label{table:Sensitivity:beta,TI,Cost}
\end{table}

\begin{table}[htbp]
\centering
\caption{MSE ($\times 10^3$) when using O-RFGP with different values of $D$ to estimate the time-invariant density. Associated with Fig.~\ref{Sensitivity:beta,TI,MSE}.}
\begin{tabular}{lcccc}
\toprule
\multicolumn{1}{c}{$\sqrt{\beta_t}$} & \textbf{$T=50$} & \textbf{$T=100$} & \textbf{$T=200$} & \textbf{$T=500$} \\ 
\midrule
\textbf{$10^1$} 
& $0.56 \pm 0.06$ 
& $0.24 \pm 0.08$ 
& $0.21 \pm 0.07$ 
& $0.19 \pm 0.05$ \\
\textbf{$10^3$} 
& $0.59 \pm 0.08$ 
& $0.21 \pm 0.04$ 
& $0.21 \pm 0.01$ 
& $0.20 \pm 0.01$ \\
\textbf{$10^5$} 
& $4.76 \pm 0.53$ 
& $0.54 \pm 0.13$ 
& $0.20 \pm 0.04$ 
& $0.17 \pm 0.03$ \\
\textbf{$\mathrm{Step}$} 
& $0.50 \pm 0.04$ 
& $0.28 \pm 0.19$ 
& $0.17 \pm 0.07$ 
& $0.13 \pm 0.03$ \\
\bottomrule
\end{tabular}
\label{table:Sensitivity:beta,TI,MSE}
\end{table}

\begin{figure}[htbp]
\centering
\subfigure[]{
\begin{minipage}[b]{0.35\textwidth}
\includegraphics[width=1\textwidth]{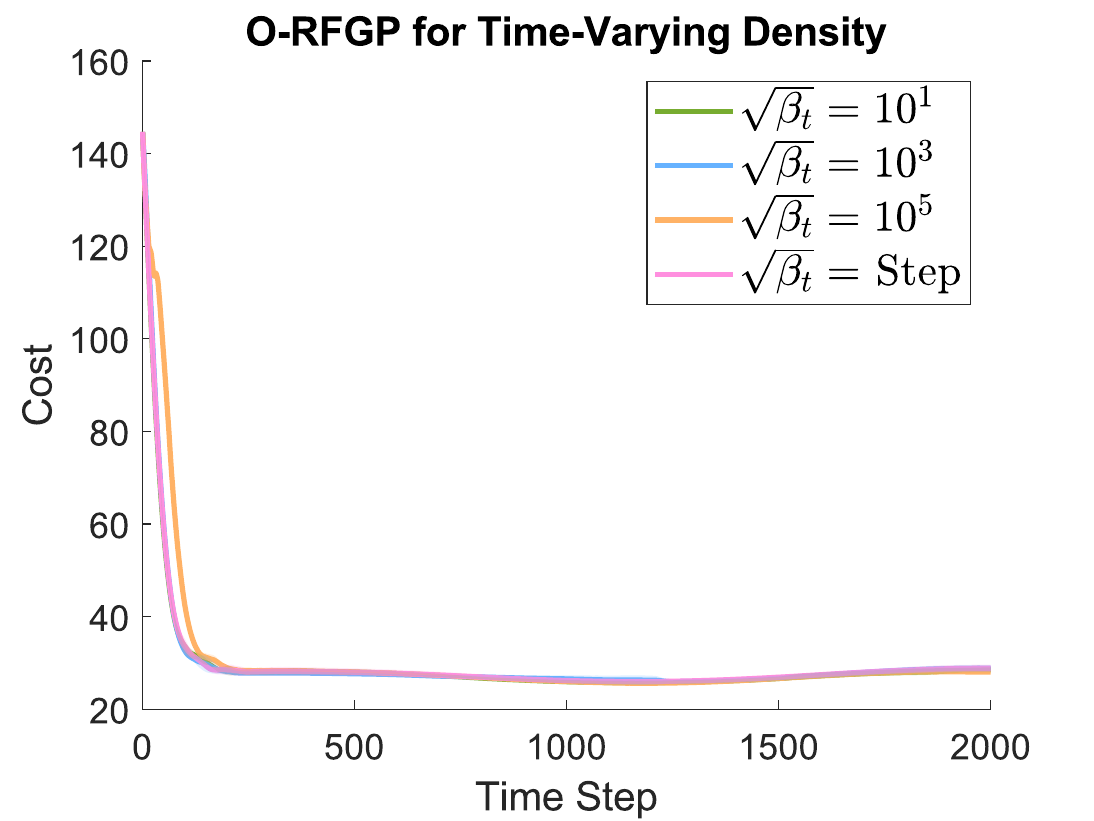}
\end{minipage}
\label{Sensitivity:beta,TV,Cost}
}
\subfigure[]{
\begin{minipage}[b]{0.35\textwidth}
\includegraphics[width=1\textwidth]{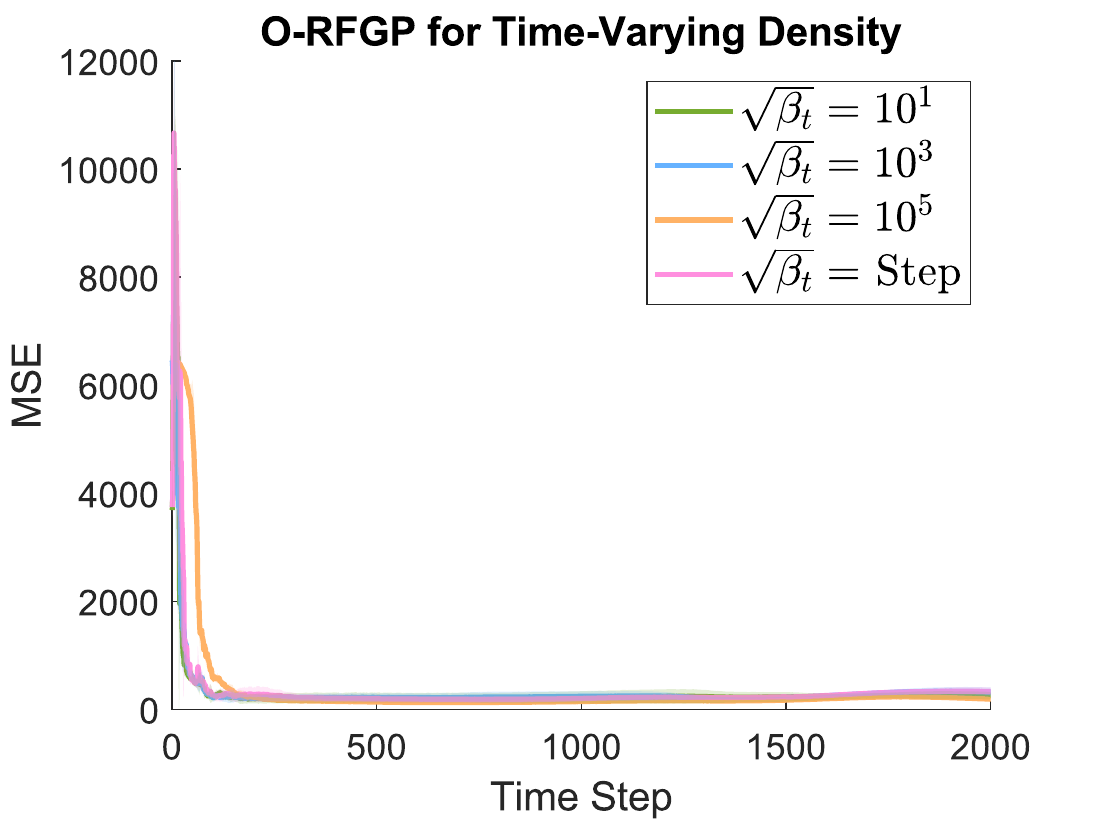}
\end{minipage}
\label{Sensitivity:beta,TV,MSE}
}
\caption{Sensitivity study of $\sqrt{\beta_t}$ when using O-RFGP to estimate the time-varying density. ``$\sqrt{\beta_t}$ = Step'' denotes a piecewise constant schedule: $\sqrt{\beta_t} = 10^1$ for time steps $0$-$500$, $\sqrt{\beta_t} = 10^2$ for time steps $500$-$1000$, $\sqrt{\beta_t} = 10^3$ for time steps $1000$-$1500$, and $\sqrt{\beta_t} = 10^4$ for time steps $1500$-$2000$.
}
\label{Sensitivity:beta,TV}
\vspace{-0.2cm}
\end{figure}

\begin{table}[htbp]
\vspace{-0.2cm}
\centering
\caption{Cost when using O-RFGP with different values of $D$ to estimate the time-varying density. Associated with Fig.~\ref{Sensitivity:beta,TV,Cost}.}
\begin{tabular}{lcccccc}
\toprule
\multicolumn{1}{c}{$\sqrt{\beta_t}$} 
& \textbf{$T=50$} 
& \textbf{$T=100$} 
& \textbf{$T=500$}
& \textbf{$T=1000$}
& \textbf{$T=1500$}
& \textbf{$T=2000$}
\\ 
\midrule
\textbf{$10^1$} 
& $59.13 \pm 0.45$  
& $33.45 \pm 1.28$  
& $28.02 \pm 0.59$  
& $25.99 \pm 0.35$ 
& $26.68 \pm 0.41$ 
& $28.56 \pm 0.52$  
\\
\textbf{$10^3$} 
& $60.60 \pm 0.49$  
& $32.76 \pm 0.79$  
& $27.68 \pm 0.22$  
& $26.58 \pm  0.78$ 
& $26.77 \pm 0.52$ 
& $28.84 \pm 0.71$ 
\\
\textbf{$10^5$} 
& $97.14 \pm 5.16$  
& $42.79 \pm 1.22$  
& $28.09 \pm 0.13$  
& $26.01 \pm 0.21$ 
& $26.68 \pm 0.21$ 
& $28.04 \pm 0.27$ 
\\
\textbf{$\mathrm{Step}$} 
& $60.19 \pm 1.37$  
& $33.71 \pm 1.62$  
& $28.02 \pm 0.83$  
& $26.22 \pm 0.51$ 
& $26.94 \pm 0.53$ 
& $28.86 \pm 0.78$ 
\\
\bottomrule
\end{tabular}
\label{table:Sensitivity:beta,TV,Cost}
\end{table}

\begin{table}[htbp]
\vspace{-0.2cm}
\centering
\caption{MSE ($\times 10^3$) when using O-RFGP with different values of $D$ to estimate the time-varying density. Associated with Fig.~\ref{Sensitivity:beta,TV,MSE}.}
\begin{tabular}{lcccccc}
\toprule
\multicolumn{1}{c}{$\sqrt{\beta_t}$} 
& \textbf{$T=50$} 
& \textbf{$T=100$} 
& \textbf{$T=500$}
& \textbf{$T=1000$}
& \textbf{$T=1500$}
& \textbf{$T=2000$}
\\ 
\midrule
\textbf{$10^1$} 
& $0.59 \pm 0.12$  
& $0.26 \pm 0.10$  
& $0.22 \pm 0.09$  
& $0.22 \pm 0.11$ 
& $0.23 \pm 0.08$ 
& $0.25 \pm 0.07$  
\\
\textbf{$10^3$} 
& $0.58 \pm 0.09$  
& $0.21 \pm 0.06$  
& $0.22 \pm 0.08$  
& $0.24 \pm 0.09$ 
& $0.22 \pm 0.004$ 
& $0.32 \pm 0.09$ 
\\
\textbf{$10^5$} 
& $5.22 \pm 0.88$  
& $0.57 \pm 0.07$  
& $0.14 \pm 0.02$  
& $0.14 \pm 0.03$ 
& $0.17 \pm 0.01$ 
& $0.19 \pm 0.04$ 
\\
\textbf{$\mathrm{Step}$} 
& $0.54 \pm 0.05$  
& $0.27 \pm 0.05$  
& $0.20 \pm 0.07$  
& $0.20 \pm 0.05$ 
& $0.23 \pm 0.04$ 
& $0.33 \pm 0.08$ 
\\
\bottomrule
\end{tabular}
\label{table:Sensitivity:beta,TV,MSE}
\vspace{-0.2cm}
\end{table}

To analyze the influence of the non-decreasing sequence $\beta_t$, the results in Fig.~\ref{Sensitivity:beta,TI}, Table~\ref{table:Sensitivity:beta,TI,Cost}, Table~\ref{table:Sensitivity:beta,TI,MSE}, Fig.~\ref{Sensitivity:beta,TV}, Table~\ref{table:Sensitivity:beta,TV,Cost} and Table~\ref{table:Sensitivity:beta,TV,MSE} show that ORC maintains sub-linear regret across a wide range of $\beta_t$ values. Although a smaller $\beta_t$ value yields slightly lower MSE, we select $\sqrt{\beta_t}=10^5$ in most simulations as a larger value provides more precise boundary visualization in heatmaps.

\subsection{Computational complexity}
\label{sec:computationalcomplexity}

\begin{figure}[h]
    \centering
    \includegraphics[scale=0.35]{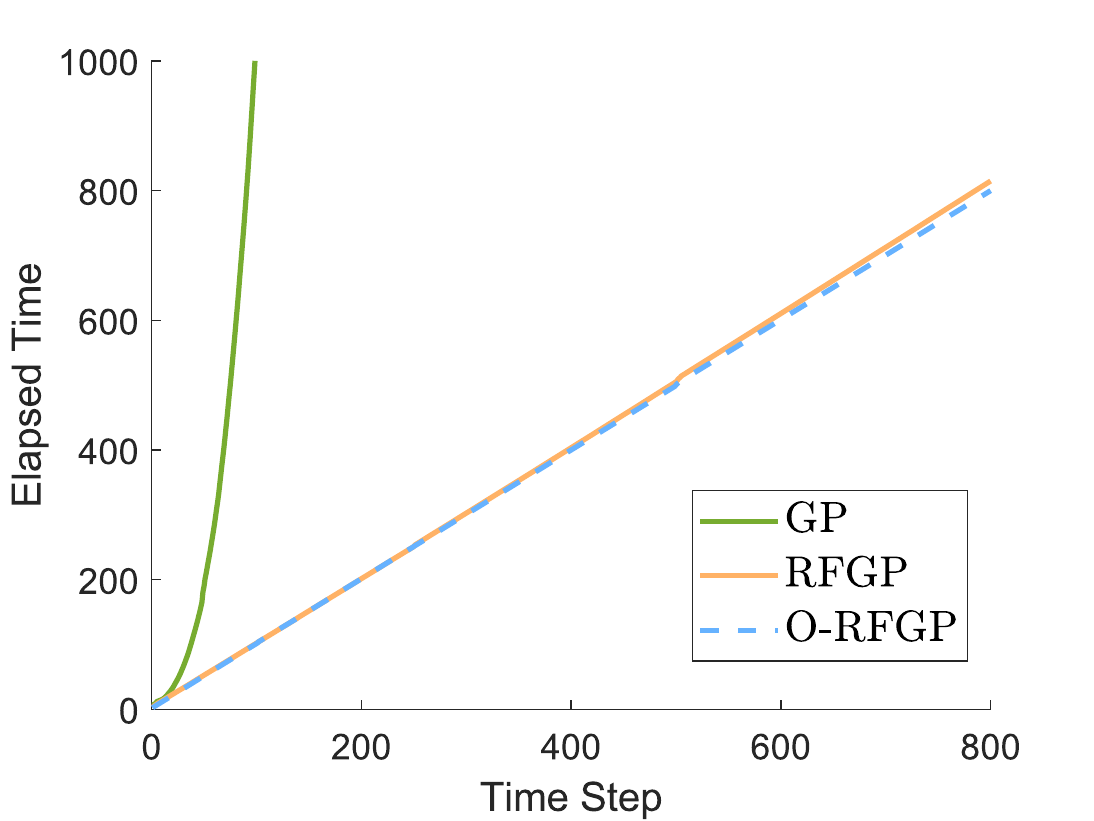}
    \caption{Comparison of the computational efficiency among GP, RFGP, and O-RFGP. The unit of the total elapsed time is the average value of the per-iteration running time of O-RFGP.}
    \label{runningtime}
\end{figure}

Fig.~\ref{runningtime} shows that the running time of GP is significantly longer than those of RFGP and O-RFGP. At around time step 200, both RFGP and O-RFGP achieve roughly a thousandfold improvement in computational efficiency over GP.
The results verifies the theoretical results presented in Table~\ref{table:computation_complex_memory_cost}. Notably, due to its lightweight computation, unlike the standard GP, O-RFGP is well-suited for real-time implementation on physical robotic platforms, as highlighted in Section~\ref{Section:Introduction} and demonstrated in Section~\ref{sec:physical experiment}.

\begin{figure}[h]
    \centering
    \includegraphics[scale=0.35]{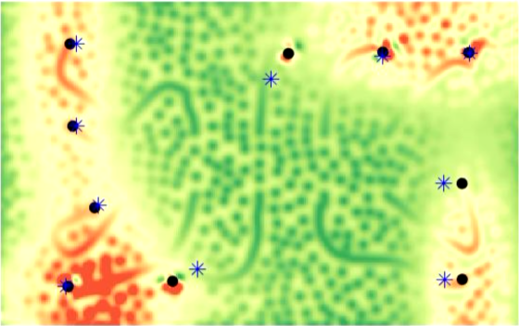}
    \caption{The learned time-varying density using GP (with all measurements collected from the beginning up to the current time step leveraged) at time step $760$, where the black dots represent the positions of robots and the blue asterisks represent the centers of mass of their Voronoi cells.} 
    \label{GP_TV_alldata_iter760}
    \vspace{-3mm}
\end{figure}
In addition, we conduct a case study using GP, where the robots are ignorant of the time-dependency of the density function. Thus, all measurements collected from the beginning of the experiment up to the current time step are leveraged for density learning. In addition to the computational burden of GP, we find in a simulation that the learning result starts to diverge at around time step $760$, as shown in Fig.\ref{GP_TV_alldata_iter760}. 
Although the density learning error of the GP decreases as the number of training samples increases, as analyzed in Section \ref{sec:regret}, the number of mismatched data points increases due to the time-dependency of the density. This result indicates the limitations of the methods proposed in \citet{srinivas2012information, williams2006gaussian, calandriello2019gaussian} in time-varying scenarios. 
Overall, as discussed in Section~\ref{Section:Introduction}, GP is not suitable for real-time applications due to its high computational complexity, consistent with Table~\ref{table:computation_complex_memory_cost} and Section~\ref{sec:computationalcomplexity}.

\end{appendices}

\end{document}